\documentclass[11pt,letterpaper]{article}
\usepackage{booktabs}
\usepackage[letterpaper,margin=1in]{geometry}

\usepackage[utf8]{inputenc}
\usepackage[T1]{fontenc}
\usepackage{lmodern}
\usepackage{mathpazo}
\usepackage{cabin}
\usepackage{bm}

\usepackage{booktabs}
\usepackage{mdframed}

\usepackage{dsfont}

\usepackage{amsfonts,nicefrac}
\usepackage[utf8]{inputenc}
\usepackage[dvipsnames]{xcolor}
\usepackage{graphicx}
\usepackage{multirow}
\usepackage{amsmath}
\usepackage{amsthm}
\usepackage{amssymb}
\usepackage{comment}
\usepackage{cancel}

\usepackage[shortlabels]{enumitem}
\usepackage[style=alphabetic, maxbibnames=15, giveninits=true, maxcitenames=15, natbib=true,
  maxalphanames=10, backend=biber, sorting=nty, backref=true, uniquename=false]{biblatex}
\DefineBibliographyStrings{english}{backrefpage = {page},backrefpages = {pages},}
\DeclareNameAlias{sortname}{given-family}

\makeatletter
\newtheorem{theorem}{Theorem}[section]

\newtheorem{lemma}[theorem]{Lemma}

\newtheorem{assumption}{Assumption}[section]
\theoremstyle{remark}
\newtheorem{remark}{Remark}[section]

\usepackage{bm}
\usepackage{hyperref}
\hypersetup{
  colorlinks   = true, %
  urlcolor     = blue, %
  linkcolor    = blue, %
  citecolor   = red %
}

\usepackage[utf8]{inputenc} %
\usepackage[T1]{fontenc}    %
\usepackage{url}            %
\usepackage{booktabs}       %
\usepackage{amsfonts}       %
\usepackage{nicefrac}       %
\usepackage{microtype}      %
\usepackage{xcolor}         %
\usepackage{xspace}
\usepackage{caption}
\usepackage{subcaption}
\usepackage{wrapfig}

\usepackage{bm}
\usepackage{amssymb}

\usepackage[linesnumbered,ruled,vlined]{algorithm2e}
\usepackage{amsmath,amsthm}

\usepackage[capitalize,noabbrev]{cleveref}
\crefname{equation}{Eq.}{Eqs.}

\usepackage[shortlabels]{enumitem}

\newcommand{\debug}[1]{#1}
\newcommand{\method}{\debug{\textsc{GALA}}\xspace}
\newcommand{\sgdmethod}{\debug{\textsc{SGD-GALA}}\xspace}
\newcommand{\adammethod}{\debug{\textsc{Adam-GALA}}\xspace}

\newcommand\Vector[1]{\mathbf{#1}}

\newcommand\vd{{\Vector{d}}}
\newcommand\ve{{\Vector{e}}}

\newcommand\vg{{\Vector{g}}}

\newcommand\vm{{\Vector{m}}}

\newcommand\vs{{\Vector{s}}}

\newcommand\vu{{\Vector{u}}}
\newcommand\vv{{\Vector{v}}}
\newcommand\vw{{\Vector{w}}}
\newcommand\vx{{\Vector{x}}}
\newcommand\vy{{\Vector{y}}}

\newcommand\bigO{\mathcal{O}}

\DeclareMathOperator*{\E}{\mathbb{E}}

\DeclareMathOperator*{\argmin}{arg\,min}

\newcommand{\reals}{\mathbb{R}}

\makeatother

\renewcommand{\epsilon}{\varepsilon}
\definecolor{comment}{RGB}{2,128, 9}

\usepackage{eqparbox}

\begin{document}

\title{%
Online Learning-guided Learning Rate Adaptation via Gradient Alignment\footnotetext{\llap{\textsuperscript{1}}The authors are listed in alphabetical order.}
}

\author{Ruichen Jiang\thanks{Department of Electrical and Computer Engineering, The University of Texas at Austin, Austin, TX, USA  \{rjiang@utexas.edu, kavis@austin.utexas.edu, mokhtari@austin.utexas.edu\}} \and Ali Kavis$^*$ \and Aryan Mokhtari$^*$ }

\date{}
\maketitle

\begin{abstract}
    The performance of an optimizer on large-scale deep learning models depends critically on \emph{fine-tuning} the learning rate, often requiring an extensive grid search over base learning rates, schedules, and other hyperparameters. In this paper, we propose a principled framework called \method (\emph{Gradient Alignment-based Learning rate Adaptation}), which dynamically adjusts the learning rate by tracking the alignment between consecutive gradients and using a local curvature estimate. Guided by the convergence analysis, we formulate the problem of selecting the learning rate as a one-dimensional online learning problem. When paired with an online learning algorithm such as Follow-the-Regularized-Leader, our method produces a flexible, adaptive learning rate schedule that tends to increase when consecutive gradients are aligned and decrease otherwise. We establish a data-adaptive convergence rate for {normalized SGD equipped with} \method in the smooth, nonconvex setting.
   {{
   Empirically, {common optimizers such as SGD and Adam, when augmented with} \method, demonstrate robust performance across a wide range of initial learning rates
    and perform competitively without the need for tuning.}}
\end{abstract}

\section{Introduction}

Stochastic first-order (SFO) methods such as SGD~\citep{robbins1951stochastic}, AdaGrad~\citep{mcmahan2010adaptive, duchi2011adaptive}, and Adam~\citep{kingma15adam} have been the workhorse for training large-scale models due to their low computational overhead and strong empirical performance. 
Essentially, the practical performance of SFO methods relies on two components: the choice of base learning rate and how the learning rate evolves during training. 
The initial selection process is typically done by running a grid search over a range of values, which is referred to as \emph{tuning}. 
On top of that, the evolution of the learning rate throughout the execution is most commonly done by scaling it externally via a \emph{scheduler}.
Depending on the characteristics of the optimizer, the learning rate could also be dynamically updated by some internal mechanism during training. 

For instance, SGD is often run with a \textit{constant base} learning rate and coupled with a scheduler such as cosine annealing~\citep{loshchilov2017sgdr}, linear decay~\citep{defazio2023optimal} or step decay~\citep{ge2019step} that guides the learning rate following a \textit{predetermined} rule. 
Similarly, the so-called adaptive methods update the learning rate internally by accumulating the observed gradients with respect to a prescribed rule that usually tends its value below its initialization. Although optimizers have other parameters such as momentum and weight decay, they are often fixed at the beginning of the execution, whereas the learning rate evolves throughout the optimization process and thus has a larger impact on the final performance.

However, it is unclear how one could choose an ``empirically viable'' combination of base learning rate, optimizer, and scheduler, \emph{a priori}, without tuning over a manually chosen set of possibilities. In this paper, we study a theoretically principled approach to learning rate adaptation that is robust, flexible, and provable. To put things in perspective, we highlight three key shortcomings of common approaches in this domain.
\begin{figure}[t]
\begin{subfigure}{0.49\linewidth}
\includegraphics[width=0.99\linewidth]{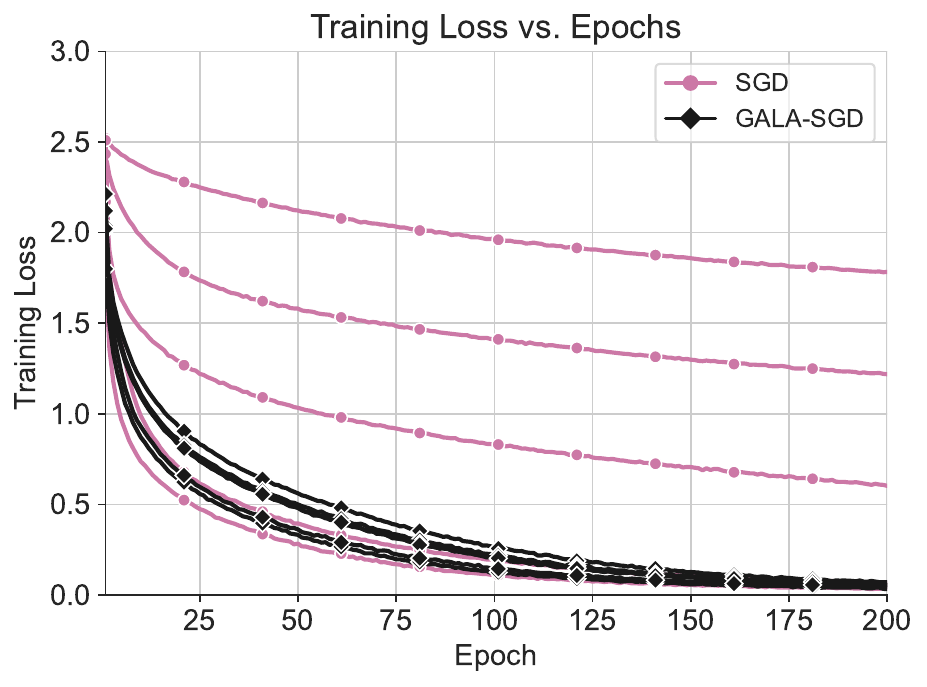}
\caption{SGD versus \sgdmethod}
\label{fig:sgd_vs_gala}
\end{subfigure}\hfill
\begin{subfigure}{0.49\linewidth}
\includegraphics[width=0.99\linewidth]{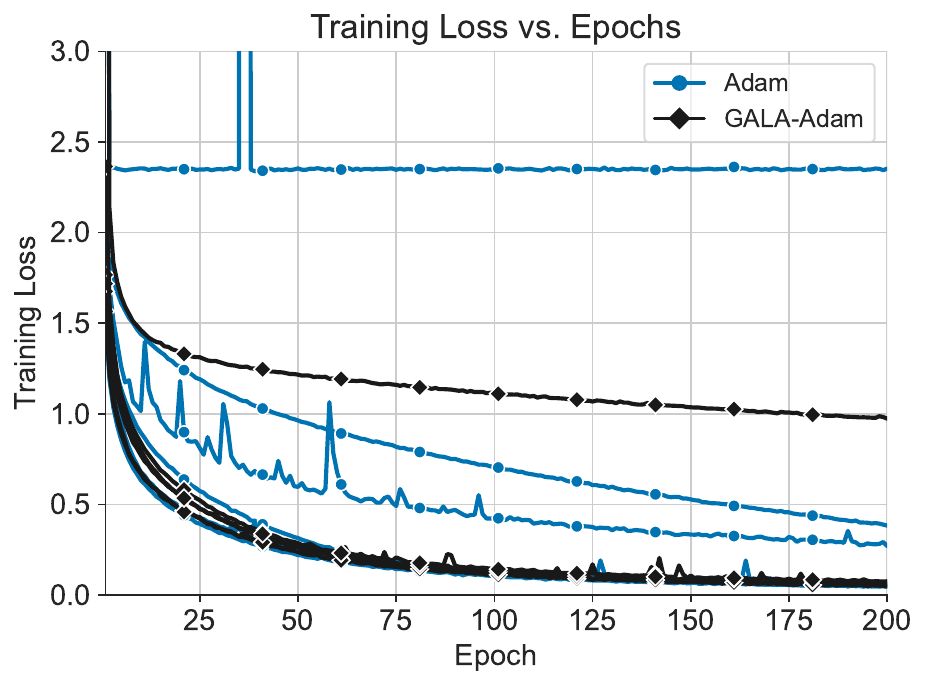}
\caption{Adam versus \adammethod}
\label{fig:adam_vs_gala}
\end{subfigure}
\caption{\label{fig:intro_plot} Training loss comparison for standalone SGD and Adam versus \method applied on their learning rates (\sgdmethod and \adammethod, respectively). The curves are obtained by running the algorithms with initial learning rates $[1, 10^{-1}, 10^{-2}, 10^{-3}, 10^{-4}, 10^{-5}]$.}
\end{figure}
\begin{enumerate} %
    \item {\textbf{Robustness:} Many optimizers, such as SGD and AdaGrad, are highly sensitive to the initial learning rate: an excessively large value can lead to divergence, while a very small one results in stagnation. 
    Ideally, we seek stable performance across a wide range of initial learning rates, allowing for robust training even with suboptimal initializations.}
    \item 
    {\textbf{Adaptivity:} In most cases, the value of the learning rate, either through internal dynamics or an external scheduler, tends to decay over time, limiting flexibility. For instance, the standard AdaGrad algorithm~\cite{mcmahan2010adaptive, duchi2011adaptive} reduces the learning rate below its initial value, and most commonly used schedulers induce decaying behavior on top of the learning rate. 
    A more desirable alternative would be a principled and adaptive scheduling mechanism capable of both increasing and decreasing the learning rate as needed.
    }
    \item {\textbf{Principled:} Most theoretical frameworks in this area are grounded in convex optimization. For instance, AdaGrad and its variants are supported by data-dependent regret bounds derived from online convex optimization. However, the decaying nature of their learning rate is not necessarily empirically optimal for non-convex landscapes.
    Developing a theoretically grounded approach tailored to non-convex problems is crucial for establishing provable performance guarantees.}

\end{enumerate}
While several existing works partially address these three limitations, none meet all the desired criteria simultaneously, as discussed in detail in Section~\ref{sec:related}.

{\color{green}{

}}

\paragraph{Our contributions:} 
Our goal is to unify all three ingredients in a single principled framework. To this end, we propose \textbf{G}radient \textbf{A}lignment-based \textbf{L}earning rate \textbf{A}daptation (\method), an online-learning guided framework for adjusting the learning rate on-the-fly by carefully monitoring the evolution of the optimization path. In particular:
\begin{enumerate} %
    \item Motivated by the convergence analysis, we construct a one-dimensional online loss function for selecting the learning rate, using the \emph{alignment between consecutive stochastic gradients} and a \emph{local estimate of the gradient Lipschitz constant}. The learning rate is then updated by performing a step for the one-dimensional online problem via any suitable algorithm.
    \item Our approach enables dynamic learning rate adaptation: it tends to \emph{increase} when the gradients are aligned and \emph{decrease} when misalignment is detected. We carefully moderate the alignment signal using a regularization term based on the local Lipschitz estimate, promoting {stability}.
    \item Theoretically, we provide a regret-based analysis and {establish convergence guarantees} for a variant of our algorithm for nonconvex objective functions with stochastic gradients.
    \item Empirically, our method demonstrates strong, stable performance across a wide range of hyperparameter settings. 
    To complement our theory-driven framework, we propose a heuristic implementation for other SFO methods.
\end{enumerate}
In fact, \cref{fig:intro_plot} provides a glimpse at the performance of our framework when applied on the learning rates of SGD and Adam. We show that \method helps mitigate sensitivity to the initialization of the learning rate while maintaining a competitive performance with respect to the best-performing runs of the standalone SGD and Adam.

\subsection{Related work}\label{sec:related}
In this section, we provide a comprehensive review of the related work in the context of our paper. 

\paragraph{Classical stochastic first-order methods}
Dating back to the seminal work~\cite{robbins1951stochastic}, the theoretical behavior of SGD and its many variants have been extensively studied. 
Considering general smooth functions, it is well-known that the learning rate must decrease at a rate of $\eta_t = O(1/\sqrt{t})$ where $t$ is the iteration counter and also satisfy $\eta_t \leq O(1 / L)$.
\citet{ghadimi2013stochastic} established that SGD with a properly chosen learning rate achieves a complexity of $\mathcal{O}({\epsilon^{-2}} + {\sigma^2}\epsilon^{-4})$, interpolating between deterministic and stochastic rates and matching the lower bounds~\citep{arjevani2023lower}. 
However, the choice of the learning rate depends on the problem parameters, i.e., $L$, $\sigma$, which are typically unknown and prohibitively difficult to estimate in practice. 
Similar requirements are in place for the learning rate when the objective function is $\rho$-weakly convex~\citep{davis2020stochastic}. 

\paragraph{Adaptive and parameter-free optimization methods}
AdaGrad was introduced in two concurrent works~\cite{mcmahan2010adaptive,duchi2011adaptive}
for minimizing a sequence of \emph{online} convex losses. The main idea is to compute a time-varying learning rate by accumulating squared norms of stochastic gradients.
This fundamental idea paved the way for many algorithms such as Adam~\citep{kingma15adam}, RMSProp~\cite{tieleman2012divide}, Adadelta~\cite{zeiler2012adadelta}, and their variants, which 
demonstrate strong empirical performance. Beyond the online optimization setup, they have been shown to automatically adapt to problem-dependent parameters such as smoothness, noise variance, and bounds on gradients. 
{Their convergence properties have been well-studied for the convex setting~\citep{levy2017online, levy2018online, kavis2019universal, joulani2020simpler, AVCL+22, liu2023on, rodomanov2024universal}} and non-convex setting \citep{li2019convergence, ward2020adagrad, li2020high, kavis2021high, gadat2022asymptotic, faw2022power, attia2023sgd, liu2023high}. However, a downside of these first-generation adaptive methods is the sensitivity to initial learning rate, dampening the practical benefits of their data-adaptive design.

{To remedy this, \emph{parameter-free} optimization \citep{carmon2022making, ivgi2023dog, khaled2023dowg, kreisler2024accelerated, attia2024parameterfree} has gained popularity with a focus on augmenting robustness. Essentially, they multiply AdaGrad-type learning rate with a scaling factor that iteratively improves the initial learning rate estimate. Although this helps increase from the initial value, {the scaling factor is practically bounded}, restricting flexibility.}
On a related front, a different line of work~\cite{malitsky2020adaptive, malitsky2024adaptive, li2024simple} study parameter-free gradient methods with local curvature estimation for convex, deterministic problems. {They are separated from AdaGrad-type methods with non-monotone learning rate that estimates time evolution of local smoothness. A downside to these methods is empirical stability; when the increasing behavior is not tamed properly, optimization performance might be unstable 
especially for nonconvex problems. Therefore, it is of utmost importance to strike the right balance between flexibility and stability.}

\paragraph{Hypergradient descent} Originally proposed as a heuristic for stochastic optimization in \cite{almeida1999parameter}, hypergradient descent updates the learning rate by computing the gradient with respect to the learning rate itself. This idea was later rediscovered and updated to modern deep learning by~\cite{rubio2017convergence,baydin2018online}, with several subsequent works refining this approach~\cite{chandra2022gradient,ozkaramada2024}. More recently, the authors in~\cite{gao2024gradient,chu2025provable} provided convergence guarantees from an online learning perspective, though their analysis is limited to deterministic convex settings. 

\paragraph{Online learning-guided methods}
Drawing insights from parameter-free online learning~\cite{orabona2016coin}, \citet{orabona2017training} reformulate SGD as a coin-betting game and apply a betting algorithm to eliminate the need for a manually tuned learning rate. They also provide convergence guarantees for convex and quasi-convex objectives. \citet{cutkosky2023mechanic} proposed a general technique for adaptively scaling any base optimization algorithm and learning rate schedule, which is grounded in a black-box reduction framework from parameter-free online learning~\cite{cutkosky2018black}. {The work most relevant to ours is that of} \citet{Zhuang2019Surrogate}, who consider non-convex stochastic optimization and introduce a surrogate loss technique for selecting the learning rate.  {However, their method requires knowledge of problem-dependent parameters (e.g., gradient's Lipschitz constant), which limits its flexibility.}

\section{Preliminaries}\label{sec:prelim}
We consider the stochastic optimization problem 
\begin{equation*}
    \min_{\vx \in \reals^d}\; F(\vx) = \mathbb{E}_{\xi \sim \mathcal{D}}[f(\vx; \xi)],
\end{equation*}
where $f(\cdot; \xi)$ is a random function indexed by a random variable $\xi$ drawn from distribution $\mathcal{D}$. The objective function 
$F: \reals^d \rightarrow \reals$ is assumed to be differentiable, possibly nonconvex and bounded from below, i.e., $F(\vx) > -\infty$. Moreover, we make the following two assumptions, which are standard in the optimization literature. 
\begin{assumption} \label{assum:smoothness}
The gradient of $F$ is $L$-Lipschitz continuous, i.e., 
$    \| \nabla F(\vx) - \nabla F(\vy) \| \leq L \| \vx - \vy \|$ for any $\vx$ and $\vy$.
\end{assumption}

\begin{assumption}\label{assum:variance}
The stochastic gradient has bounded variance of $\sigma^2$, i.e., $\E[\|\nabla F(\vx) - \nabla f(\vx; \xi)\|^2] \!\leq \! \sigma^2$ for any $\vx \in \reals^d$. 
\end{assumption}

\subsection{Background: online learning}
Let us briefly introduce the online learning framework and establish the groundwork necessary within the context of our approach. In the online learning framework, a learner makes decisions iteratively over rounds. At each round $t = 1, \cdots, T$:
\begin{enumerate}
    \item The learner makes a decision $\vx_t\in \mathcal X$ from a bounded set of actions;
    \item The environment/adversary reveals the loss function $\ell_t(\cdot)$;
    \item The learner suffers the loss $\ell_t(\vx_t)$.
\end{enumerate}
The learner can use the history of decisions and losses to make a new decision per round. The learner chooses its action $\vx_t$ in round $t$ \emph{prior to} observing the loss $\ell_t(\cdot)$. The performance of the learner is measured by \emph{regret}, which is defined as the difference between the cumulative loss of the learner compared against a fixed action $\vx$:
\begin{equation} \label{eq:regret}
    \mathrm{Reg}_T(\vx) = \sum_{t=1}^{T} (\ell_t(\vx_t) - \ell_t(\vx)).
\end{equation}
The goal is to achieve \emph{sublinear regret}, i.e., $\mathrm{Reg}_T(\vx) = o(T)$, such that the time average of regret goes to zero as $T \to \infty$, meaning the learner performs as well as the fixed strategy in the limit.

\section{Online learning rate selection}\label{sec:online_learning}
We begin by introducing a simplified template that outlines the fundamentals of our design. Our primary goal is to provide insight into the idea of gradient alignment, explain our adaptive strategy, and establish the foundation for the online learning formulation of the learning rate. 

Consider the SGD update rule 
\begin{equation}\label{eq:SGD}
    \vx_{t+1} = \vx_t - \eta_t \vg_t(\vx_t), \quad \vg_t(\vx_t) = \nabla f(\vx_t; \xi_t),
\end{equation}
where $\xi_t\sim \mathcal{D}$ is a random sample drawn from the distribution $\mathcal{D}$ at iteration $t$.
Our goal is to choose a sequence of learning rates guided by the progress of the algorithm, as measured by the function value difference $F(\vx_{t+1}) - F(\vx_t)$. At this point, we deviate from the classical analysis; inspired by~\cite{cutkosky2023optimal}, we apply the fundamental theorem of calculus to get  
\begin{equation}\label{eq:FTC}
    F(\vx_{t+1}) - F(\vx_t) = \langle \bm{\nabla}_t, \vx_{t+1} - \vx_t \rangle = -\eta_t \langle \bm{\nabla}_t, \vg_t(\vx_t) \rangle,
\end{equation}
where $\bm{\nabla}_t = \int_{0}^1 \nabla F(\vx_t + \lambda(\vx_{t+1}-\vx_t))\,d\lambda$ denotes the average gradient along the line segment between $\vx_t$ and $\vx_{t+1}$. Note that the right-hand side of \eqref{eq:FTC} concerns the \emph{alignment} between the gradients~$\bm{\nabla}_t$ and $\vg_t(\vx_t)$ and serves as a useful signal for adjusting the learning rate. When the alignment term is positive, it indicates that the gradients point in similar directions and increasing the learning rate may lead to greater progress. Conversely, a negative alignment implies opposing directions, in which case a smaller learning rate may be more appropriate.  

However, computing $\bm{\nabla}_t$ is generally intractable, as it involves the true gradient and an integral. A key observation in~\cite{zhang2020complexity,cutkosky2023optimal} is that an unbiased estimate of $\bm{\nabla}_t$ can be constructed by evaluating the gradient at a random point along the line segment. Specifically, let $\lambda_t$ be a random variable uniformly distributed over $[0,1]$, and let $\xi_t'$ be an independent sample from the distribution $\mathcal{D}$. Then for $\vw_{t} = \vx_t + \lambda_t(\vx_{t+1} - \vx_t)$ and $\vg_t'(\vw_t) = \nabla f(\vw_t; \xi_t')$, we have $\bm{\nabla}_t = \E_{\lambda_t}[\nabla F(\vw_t)] = \E_{\lambda_t, \xi_t'}[\vg_t'(\vw_t)]$, which implies
\begin{equation}\label{eq:integral}
    \textstyle F(\vx_{t+1}) - F(\vx_t) = %
    -\eta_t\E_{\lambda_t,\xi_t'}[\langle \vg_t'(\vw_t), \vg_t(\vx_t)\rangle]. 
\end{equation}
To maximize the decrease in the function value, \cref{eq:integral} suggests that a natural objective is to minimize $-\eta_t\E_{\lambda_t,\xi_t'}[\langle \vg_t'(\vw_t), \vg_t(\vx_t)\rangle]$. However, this approach comes with two issues.
Let us begin with the first point, which is related to the convergence metric. This approach only leads to an upper bound on $\frac{1}{T}\sum_{t=0}^{T-1} \E[\langle \vg_t'(\vw_t), \vg_t(\vx_t)\rangle]$, which does not directly provide a meaningful bound on gradient norm of $F$, which is the standard metric we would like to obtain. Our idea is to decompose the inner product as $\langle \vg_t'(\vw_t), \vg_t(\vx_t)\rangle = \langle \vg_t'(\vx_t), \vg_t(\vx_t)\rangle + \langle \vg_t'(\vw_t) - \vg_t'(\vx_t) , \vg_t(\vx_t) \rangle$, where $\vg_t'(\vx_t) = \nabla f(\vx_t; \xi_t)$. Note that the first term leads to $\E[ \langle \vg_t'(\vx_t), \vg_t(\vx_t)\rangle] = \E[\|\nabla F(\vx_t)\|^2]$ and the second term can be controlled using the Lipschitz constant of the gradient. 

The second issue is that the minimization of the right-hand side of \cref{eq:integral} with respect to the learning rate $\eta_t$ is an \emph{implicit} problem. The objective depends on the interpolated point $\vw_t$, which could be determined only \emph{after} the learning rate $\eta_t$ is chosen. The solution is to cast the learning rate selection as an \emph{online learning problem}, and derive a sequence of online loss functions that will govern the selection process. We combine and formalize both ideas in the following lemma.
\begin{lemma}\label{lem:regret}
    Define the local Lipschitz estimate $L_t = \frac{\|\vg'_t(\vw_t) - \vg'_t(\vx_t)\|}{\|\vw_t - \vx_t\|}$ and the surrogate loss function 
\begin{equation}\label{eq:surrogate_loss}
        \ell_t(\eta) \triangleq   - \eta \langle \vg_t'(\vw_t), \vg_t(\vx_t)\rangle + \frac{L_t \|\vg_t(\vx_t)\|^2\eta^2}{2}. 
    \end{equation}
    Suppose that Assumption~\ref{assum:variance} holds and $L_t \leq L^{\max}$ for any $t \geq 0$ with probability one.
    Then we have  
\begin{equation}\label{eq:sum_grad_squares}
        \sum_{t=0}^{T-1}\E[(\eta -\eta^2 L^{\mathrm{max}}) \|\nabla F(\vx_t)\|^2] \leq \E\left[ F(\vx_0) - F(\vx_T)  + L^{\mathrm{max}} \sum_{t=0}^{T-1} \eta^2 \sigma^2 + \sum_{t=0}^{T-1}(\ell_t(\eta_t) - \ell_t(\eta))\right]. 
    \end{equation}
\end{lemma}

\DontPrintSemicolon
\begin{algorithm}[!t]
\SetAlgoLined
\KwIn{Initial point $\vx_0$, initial learning rate $\eta_0$, maximum learning rate ${\eta}^{\max}$, $\delta > 0$}
\For{$t = 0$ \KwTo $T$}
{
    Sample $\xi_t \sim \mathcal{D}$ and compute $\vg_t(\vx_t) = \nabla f(\vx_t; \xi_t)$\;
    $\vx_{t+1} = \vx_{t} - \eta_t \vg_t(\vx_t)$\;
    {Sample $\xi_t' \sim \mathcal{D}$ and compute $\vg_t'(\vx_t) = \nabla f(\vx_t; \xi_t')$}\;
    Sample $\vs_t\sim \text{Uniform}[0,1]$, compute $\vw_t = \vx_t + s_t(\vx_{t+1}-\vx_t)$ and ${\vg}_t'(\vw_t) = \nabla f(\vw_t;\xi_t') $ \label{line:sampling}\;
    Compute $L_t = \frac{\|{\vg}_t'(\vw_t) - \vg_t'(\vx_t)\|}{\|\vw_t - \vx_t\|}$\;
    $\eta_{t+1} = \mathrm{clip}_{[0, {\eta}^{\max}]}\left(\frac{ \sum_{s=0}^{t}\langle \vg_s'(\vw_s), \vg_s(\vx_s) \rangle}{\delta + \sum_{s=0}^{t} L_{s}\| \vg_s(\vx_s)\|^2}\right)$
}
\caption{\sgdmethod}\label{alg:GD_FTL}
\end{algorithm}

As shown in \eqref{eq:surrogate_loss}, our surrogate loss function $\ell_t(\eta)$ consists of two terms. The first term measures the \emph{alignment} between two consecutive (stochastic) gradients $\vg_t'(\vw_t)$ and $\vg_t(\vx_t)$, and the second term is a quadratic regularization term that depends on our local estimate of the Lipschitz constant $L_t$. The online nature of the problem is due to the fact that both $\vg'(\vw_t)$ and $L_t$ can only be computed after the learning rate $\eta_t$ is chosen. Moreover, $\eta$ in \eqref{eq:sum_grad_squares} is the comparator of our online learning problem and it can be chosen \emph{arbitrarily} in our analysis. If we manage to achieve a low regret of the online learning problem (as we show in \cref{sec:convergence}), then a proper choice of $\eta$ will lead to a complexity of $\bigO(\epsilon^{-2} + \sigma^2 \epsilon^{-4})$.

\begin{remark}
    Our approach is inspired by both \cite{cutkosky2023optimal} and \cite{Zhuang2019Surrogate}. Compared to~\cite{cutkosky2023optimal}, the key difference is that their method uses online learning to guide the choice of the update direction, whereas we focus on selecting the learning rate. In contrast to~\cite{Zhuang2019Surrogate}, our method differs in two major ways: (i) we estimate the local Lipschitz constant on the fly, instead of relying on a global Lipschitz estimate; and (ii) for the first term, we use the alignment between two stochastic gradients evaluated at different points $\vw_t$ and $\vx_t$, while \citet{Zhuang2019Surrogate} use gradients at the same point.  
\end{remark}

The next step is using an online learning algorithm that will operate on the loss sequence $\ell_t$ to update our learning rate $\eta_t$. Since the loss functions are quadratic in their input, we have several options to choose from. As an example, the Follow-the-Regularized-Leader (FTRL) algorithm is given by
\begin{equation*}
    \eta_{t+1} = \argmin_{\eta \in [0, {\eta}^{\max}]} \left\{\sum_{s=0}^{t} \ell_t(\eta) + \frac{\delta}{2}\eta^2 \right\}, 
\end{equation*}
where ${\eta}^{\max}$ is the maximal learning rate and $\delta \geq 0$ is a user-defined constant to ensure stability. Using the definition of \eqref{eq:surrogate_loss} and the FTRL update, we obtain the 
following closed-form expression for~$\eta_{t+1}$: 
\begin{equation}\label{eq:eta_FTRL}
    \eta_{t+1} = \mathrm{clip}_{[0,{\eta}^{\max}]}\left(\frac{\sum_{s=0}^{t}\langle \vg_s'(\vw_{s}), \vg_s(\vx_s) \rangle}{\delta+ \sum_{s=0}^t L_{s}\|\vg_s(\vx_s)\|^2}\right), 
\end{equation}
where $\mathrm{clip}_{[0,{\eta}^{\max}]}(\cdot)$ denotes the operation that clips a real-valued input to the interval $[0,{\eta}^{\max}]$. 
For convenience, we summarize our method in Algorithm~\ref{alg:GD_FTL}. 
\begin{remark}
    The learning rate incorporates directional information along the optimization path through the alignment term: when the gradients are aligned, the learning rate is encouraged to increase; when they are misaligned, it decreases. Additionally, the quadratic regularization term moderates the learning rate update based on the magnitude of observed gradients. Note that this adaptive behavior is an inherent feature of our online learning-guided learning rate and holds by default for various choices of online learners, such as OGD~\citep{zinkevich2003online, Orabona2019}.
\end{remark}
\begin{remark}
    {For numerical stability, we pick FTRL as our choice of online learner to update the learning rate of the algorithm but FTL is also applicable since the surrogate losses are quadratic. In either case, the resulting update for the learning rate is independent of the initialization; only the very first step is taken using the base learning rate.}
\end{remark}

\section{Convergence analysis}\label{sec:convergence}
In this section, we analyze a variant of our proposed method in Algorithm~\ref{alg:GD_FTL} and establish its convergence rate for stochastic nonconvex optimization. Instead of using the standard SGD update rule in \eqref{eq:SGD}, we adopt the normalized SGD with momentum~\cite{cutkosky2020momentum}. As we will show, the main theoretical advantage of using a normalized update is that it simplifies the surrogate loss function, making the regret bound easier to establish. Specifically, we consider the update rule 
\begin{equation}\label{eq:NGD}
     \vx_{t+1} = \vx_t - \eta_t \frac{\vm_t}{\|\vm_t\|}, \qquad \vm_t = (1-\alpha)\vm_{t-1} + \alpha \nabla f(\vx_t; \xi_t),
\end{equation}
where $\alpha \in (0,1]$ is the momentum parameter. The normalization step ensures that the learning rate $\eta_t$ directly controls the distance between $\vx_{t+1}$ and $\vx_t$, thus promoting stability. However, normalization can also amplify the noise in the stochastic gradient. To mitigate this, we apply an exponential moving average in \eqref{eq:NGD}, which reduces variance and is governed by the momentum parameter $\alpha$. 

Due to the different update rule in \eqref{eq:NGD}, the surrogate loss function must be modified accordingly. Specifically, we define a new surrogate loss function as
\begin{equation}\label{eq:surrogate_loss_N}
    \ell^{\mathrm{N}}_t(\eta) = -\eta \Bigl\langle \vg_t'(\vw_t), \frac{\vm_t}{\|\vm_t\|} \Bigr\rangle + \Bigl(\frac{L_t}{2}+ \frac{4(1-\alpha) \tilde{L}_t}{3\alpha} \Bigr)\eta^2,
\end{equation}
where $\tilde{L}_t = \frac{\|\vg_t'(\vw_t) - \vg_t'(\vx_t)\|}{\|\vw-\vx_t\|}$ is a second local Lipschitz estimate. There are three main differences compared with \eqref{eq:surrogate_loss}. First, the linear term in \eqref{eq:surrogate_loss_N} measures the alignment between the gradient $\vg_t'(\vw_t)$ and the normalized update direction $\frac{\vm_t}{\|\vm_t\|}$, rather than with the stochastic gradient $\vg_t(\vx_t)$. In addition, the quadratic term is independent of the norm of the stochastic gradient $\|\vg_t(\vx_t)\|$ due to normalization step. Moreover, it includes an additional regularization term that depends on  the momentum parameter $\alpha$ and the local Lipschitz estimate $\tilde{L}_t$, which arises from the analysis of momentum. 
In the following theorem, we establish the convergence rate of the update rule in \eqref{eq:NGD} in terms of the regret with respect to the new surrogate loss in \eqref{eq:surrogate_loss_N}. The proof can be found in Appendix~\ref{appen:NGD}.

\begin{theorem}\label{thm:NGD}
    Let $\{\vx_t\}_{t=0}^{T-1}$ be the iterates following the update rule in \eqref{eq:NGD} and suppose that $\eta_t \leq {\eta}^{\max}$ for all $t$. Recall that 
    $L_t = \frac{\|\vg_t'(\vw_t) - \vg_t'(\vx_t)\|}{\|\vw-\vx_t\|}$, $\tilde{L}_t = \frac{\|\vg_t'(\vx_{t+1}) - \vg_t'(\vx_{t})\|}{\|\vx_{t+1} - \vx_t\|}$, and the surrogate loss function $\ell^{\mathrm{N}}_t(\eta)$ defined in \eqref{eq:surrogate_loss_N}. 
    Moreover, define the associated regret $\mathrm{Reg}^{\mathrm{N}}_T \triangleq \max_{\eta \in [0, {\eta}^{\max}]} \sum_{t=0}^{T-1} (\ell^{\mathrm{N}}_t(\eta_t) - \ell^{\mathrm{N}}_t(\eta))$, the initial function value gap $\Delta_F \triangleq F(\vx_0) - F(\vx^*)$, the average Lipschitz estimate $L_T^{\mathrm{avg}} = \max\Bigl\{\E\Bigl[\frac{1}{T}\sum_{t=0}^{T-1} L_t \Bigr],\E\Bigl[\frac{1}{T}\sum_{t=0}^{T-1} \tilde{L}_t\Bigr]\Bigr\}$.
    Then if we choose $\alpha = \min\{\frac{\sqrt{{{L}^{\mathrm{avg}}_T(\Delta_F + \mathrm{Reg}^{\mathrm{N}}_T)}}}{\sigma \sqrt{T}},1\}$, it holds that 
    \begin{align*}
        \frac{1}{T}\sum_{t=0}^{T-1} \E[\|\nabla F(\vx_t)\|] &= \bigO\Bigl(\frac{\sigma^{1/2}({L}^{\mathrm{avg}}_T(\Delta_F + \mathrm{Reg}^{\mathrm{N}}_T))^{1/4}}{T^{1/4}}  +\frac{\sigma^2}{\sqrt{{L}^{\mathrm{avg}}_T(\Delta_F + \mathrm{Reg}^{\mathrm{N}}_T) T}} \\
  & \phantom{{}\leq{}} +\frac{\sqrt{L_T^{\mathrm{avg}}(\Delta_F + \mathrm{Reg}^{\mathrm{N}}_T)}}{\sqrt{T}} 
  + \frac{\Delta_F + \mathrm{Reg}^{\mathrm{N}}_T}{{\eta}^{\max}T}\Bigr). 
    \end{align*}
\end{theorem}
The convergence rate of Theorem~\ref{thm:NGD} depends on the regret of the associated online learning problem. Hence, we propose to use an online learning algorithm to adaptively update the learning rate $\eta_t$ in \eqref{eq:NGD}. Note that if the local Lipschitz estimates $L_t$ and $\tilde{L}_t$ have uniform lower bounds, then the loss function in \eqref{eq:surrogate_loss_N} is strongly-convex and thus a logarithmic regret is possible.  For best theoretical guarantees, we use an optimistic variant of FTRL~\cite{rakhlin2013online,steinhardt2014adaptivity,mohri2016accelerating} given by: 
\begin{equation}\label{eq:OFTRL}
    \eta_{t+1} = \argmin_{\eta \in [0,\eta^{\max}]}\left\{\sum_{s=0}^t \ell_t(\eta) + \frac{\delta}{2}\eta^2 + h_{t+1}(\eta)\right\},
\end{equation}
where $h_{t+1}(\cdot)$ is a hint function that aims to approximate the next loss function $\ell_{t+1}$. Specifically, note that $\vm_{t+1}$ is already known at the time when we select $\eta_{t+1}$. Hence, we set $h_{t+1}(\eta) = -\eta \langle \vg_{t+1}(\vx_{t+1}), \frac{\vm_{t+1}}{\|\vm_{t+1}\|}  \rangle$, which yields the following closed-form update rule: 
\begin{equation*}
        \eta_{t+1} = \mathrm{clip}_{[0,{\eta^{\max}}]}\Bigl(\frac{\sum_{s=0}^{t}\langle \vg_s'(\vw_{s}), \vm_s/ \|\vm_s\| \rangle + \langle \vg_{t+1}(\vx_{t+1}),\vm_{t+1}/\|\vm_{t+1}\|\rangle}{\delta+ \sum_{s=0}^t (L_{s}+\frac{8(1-\alpha)}{3\alpha} \tilde{L}_s)}\Bigr). 
    \end{equation*}
We bound the regret of the above update rule in the following lemma (see Appendix~\ref{appen:regret} for the proof).
\begin{lemma}\label{lem:regret_bound}
    Let $\eta^{\max} = \sqrt{\alpha} \bar{\eta}$ for some given $\bar{\eta}$. Suppose that $\frac{\sum_{s=0}^t L_s}{t+1}  \geq M^{\mathrm{avg}}$   and $\max\{L_t, \tilde{L}_t\} \leq L^{\max}$ hold for any $t$ with probability one. Then we have
    \begin{equation*}
        \mathrm{Reg}_T^{\mathrm{N}} = {\bigO}\Bigl(\bar{\eta}^2 L^{\max}\log\Bigl(1+\frac{{L}^{\max}}{\alpha \delta}T\Bigr) + \frac{\sigma^2}{M^{\mathrm{avg}}} \log T \Bigr). 
    \end{equation*}
\end{lemma}

Combining Theorem~\ref{thm:NGD} and Lemma~\ref{lem:regret_bound}, up to logarithmic factors, we have established that our method achieves a convergence rate of $\bigO(\frac{\sigma^{1/2}}{T^{1/4}}+ \frac{1}{\sqrt{T}})$, which matches the rate in \cite{cutkosky2020momentum} with a constant learning rate. Moreover, the convergence rate in Theorem~\ref{thm:NGD} is in terms of the average and maximum Lipschitz estimates, which can be much smaller than the global Lipschitz constant $L$ in \cite{cutkosky2020momentum}. Notably, we achieve this convergence rate by adaptively selecting the learning rate instead of using a predefined constant. Finally, we remark that our convergence results are comparable to those obtained for AdaGrad~\cite{faw2022power,attia2023sgd, liu2023high},  with the key distinction that our learning rate can both increase and decrease, while the AdaGrad rate is monotonically decreasing.

\section{Numerical experiments}\label{sec:numerical}

In this section, we present preliminary results of applying \method on the image classification task of training a residual network~\cite{he2016deep} on the CIFAR-10 dataset~\cite{krizhevsky2009learning}. Before discussing these results, we first describe the experimental setup, implementation details of algorithms and the practical considerations that improve the practical performance. 

\subsection{Experimental setup: training ResNet-18}

\paragraph{Model and dataset.} We use the torchvision implementation of ResNet18 model and train it on the CIFAR-10, CIFAR-100 and Flower102 datasets. During training, we augment all three datasets with random crops and horizontal flips; we additional use color jitter on the relatively more difficult CIFAR-100 and Flower102.

\paragraph{Optimizers.} We apply \method on two popular optimizers: SGD and Adam, which we denote as \sgdmethod and \adammethod, respectively, and compare them against standalone SGD and Adam. We also include two parameter-free algorithms in our comparison: Mechanic~\citep{cutkosky2023mechanic} and AdGD~\citep{malitsky2020adaptive}.

\paragraph{Additional hyperparameters.} Since our main focus is on adaptive learning rate selection, we set all other hyperparameters to their standard default values. Specifically, for \adammethod, we fix $\beta_1 = 0.9$, $\beta_2 = 0.999$, and $\delta = 10^{-8}$, consistent with the settings used for Adam. For SGD with momentum, we set the momentum parameter to $0.9$.

\paragraph{Setup.} Starting from the same initial model parameters, we run each method with initial learning rates from the list $[1, 10^{-1}, 10^{-2}, 10^{-3}, 10^{-4}, 10^{-5}, 10^{-8}]$ and fix all other parameters at prescribed values. Following the standards in the literature, we use a training batch size of 128 and set the weight decay to zero. All the experiments are run for 200 epochs. In order to quantify the error due to randomness, we ran the experiments with three different random seeds, which we report as error bars.

\paragraph{Hardware.} Our experiments were conducted on a cluster with NVIDIA A100 GPUs (96GB memory) and 120GB system RAM. %
The CIFAR-10 experiments with multiple random seeds required approximately 96 GPU hours, and both the CIFAR-100 and the Flower102~\citep{nilsback2008automated} experiments required approximately 192 GPU hours.

\subsection{Implementation details}

\paragraph{SGD with GALA.}  When applying \method to augment the standard SGD update rule, we introduce the following three key modifications to Algorithm~\ref{alg:GD_FTL}: 
\begin{enumerate} %
    \item Instead of sampling a random point $\vw_t$ from the segment~(cf. Line~\ref{line:sampling}), we set $\vw_t = \vx_{t+1}$. 
    \item To evaluate the alignment at time $t$, we compute both gradients using the same mini-batch $\xi_{t+1}$;  
    specifically, we use the inner product $\langle \nabla f(\vx_{t+1}; \xi_{t+1}), \nabla f(\vx_{t}; \xi_{t+1}) \rangle$ as the first term of our surrogate loss function defined in~\eqref{eq:surrogate_loss}. 
    \item We omit the clipping step in the learning rate update~\eqref{eq:eta_FTRL}. 
\end{enumerate}
Among these modifications, the first and third are mainly for simplicity. In contrast, the second plays a crucial role in the empirical performance of our method, as discussed in the Appendix. 
Incorporating these changes, the learning rate update rule using FTRL becomes:  
\begin{equation}\label{eq:GD_heurstic}
    L_t = \frac{\|\nabla f(\vx_{s+1}; \xi_{s+1}) - \nabla f(\vx_s; \xi_{s+1})\|}{\|\vx_{t+1} - \vx_t\|}, \quad \eta_{t+1} = \frac{\sum_{s=0}^t \langle \nabla f(\vx_{s+1}; \xi_{s+1}), \nabla f(\vx_s; \xi_{s+1}) \rangle}{\sum_{s=0}^t L_s\|\vg_s(\vx_s)\|^2}.
\end{equation}

\paragraph{Adam with GALA.} In addition to SGD, we also adapt our \method to Adam optimizer~\cite{kingma15adam}. Specifically, the standard Adam update rule is given by  
\begin{align*}
    \vm_{t} &= \beta_1 \vm_{t-1} + (1-\beta_1)\nabla f(\vx_t; \xi_t), \\
    \vv_t &= \beta_2 \vv_{t-1} + (1 - \beta_2) \nabla f(\vx_t; \xi_t)^2, \\
    \vd_t &= \frac{\vm_t}{\sqrt{\delta + \vv_t}},\quad \vx_{t+1} = \vx_t - \eta_t  \vd_t,
\end{align*}
where all operations are element-wise, and we omit bias correction terms for simplicity. To select the learning rate $\eta_t$ for Adam, we modify the surrogate loss function in \eqref{eq:surrogate_loss} as follows: 
(i) we replace the SGD direction $\nabla f(\vx_t; \xi_t)$ with the Adam update direction $\vd_t$; (ii) we substitute $\nabla f(\vw_t; \xi_t')$ with $\nabla f(\vx_t; \xi_t)$, so that the gradient alignment term involves the inner product of stochastic gradients computed on the same mini-batch $\xi_t$; (iii) we estimate $L_t$ using the same heuristic as in \eqref{eq:GD_heurstic}. These modifications lead to the following learning rate update rule: 
\begin{equation*}
    \eta_{t+1} = \frac{\sum_{s=0}^t \langle \vd_s, \nabla f(\vx_s; \xi_{s}) \rangle}{\sum_{s=0}^t L_s\|\vg_s(\vd_s)\|^2}.
\end{equation*}

\paragraph{Mechanic.} The Mechanic algorithm, proposed in \cite{cutkosky2023mechanic}, provides a general framework for adaptively selecting the learning rate of any base optimizer. At each iteration, it proceeds as follows: 
\begin{itemize}
    \item Sample $\xi_t \sim \mathcal{D}$ and compute the stochastic gradient $\vg_t = \nabla f(\vx_t; \xi_t)$; 
    \item Use $\vg_t$ to compute the update direction $\vu_t$ via the base optimizer and update the cumulative direction $\bm{\Delta}_{t+1} = \bm{\Delta}_t + \vu_t$; 
    \item  An internal online learner selects a learning rate $s_{t+1}$; 
    \item Update the iterate: $\vx_{t+1} = \vx_1 + s_{t+1} \bm{\Delta}_{t+1}$.
\end{itemize}
For example, if the base optimizer is SGD, then $\vu_t = - \eta \vg_t$. 
In our experiments, we apply Mechanic to both SGD and its momentum variant and vary the initial learning rate $\eta$, using the official implementation available at \url{https://github.com/optimizedlearning/mechanic}.  

\paragraph{AdGD.} The update rule of AdGD in~\cite{malitsky2020adaptive} {for the stochastic setting} is given by 
\begin{equation}\label{eq:adgd}
   \begin{aligned}
    \eta_t &= \min \left\{ \sqrt{ 1 + \alpha \frac{\eta_{t-1}}{\eta_{t-2}} } \eta_{t-1}, \frac{\| \vx_t - \vx_{t-1}  \|}{2 \| \nabla f(\vx_t; \xi_t) - \nabla f(\vx_{t-1}; \xi_t) \|} \right\},\\
    \vx_{t+1} &= \vx_t - \eta_t \nabla f(\vx_t; \xi_t),
\end{aligned} 
\end{equation}
where $\alpha = 1$ in the original algorithm, which is analyzed under deterministic gradients. In practice, the authors recommend using smaller values of $\alpha$ to improve stability and avoid spikes in the loss curve. For example, they report that for ResNet-18 on CIFAR-10, setting $\alpha = 0.02$ yields the best performance. Following their recommendation, we use this value in all of our experiments.

\subsection{Experiments}
We particularly focus on the notion of robustness and study the performance of algorithms from a wide range of initial learning rate. Specifically, we report the final training loss (\cref{fig:app-trainloss}), training accuracy (\cref{fig:app-trainacc}), and testing accuracy (\cref{fig:app-testacc}) with respect to different base learning rates for each algorithm we test. The error bars show the standard deviation over three runs.

As shown in Figure~\ref{fig:sgd_vs_gala}, \sgdmethod exhibits robustness across a wide range of initial learning rates, with all configurations following similar convergence trajectories. In contrast, the performance of standard SGD is highly sensitive to the learning rates. Especially, overly small initial learning rates result in significantly slower convergence. A similar pattern is observed in Figure~\ref{fig:adam_vs_gala} when comparing \adammethod to Adam. While \adammethod maintains stable performance across most learning rates---except when initialized with a relatively large value such as 1---Adam displays greater variability and may become unstable when the learning rate is too large.
\begin{figure}[t!]
    \begin{subfigure}{0.33\linewidth}
    \includegraphics[width=0.98\linewidth]{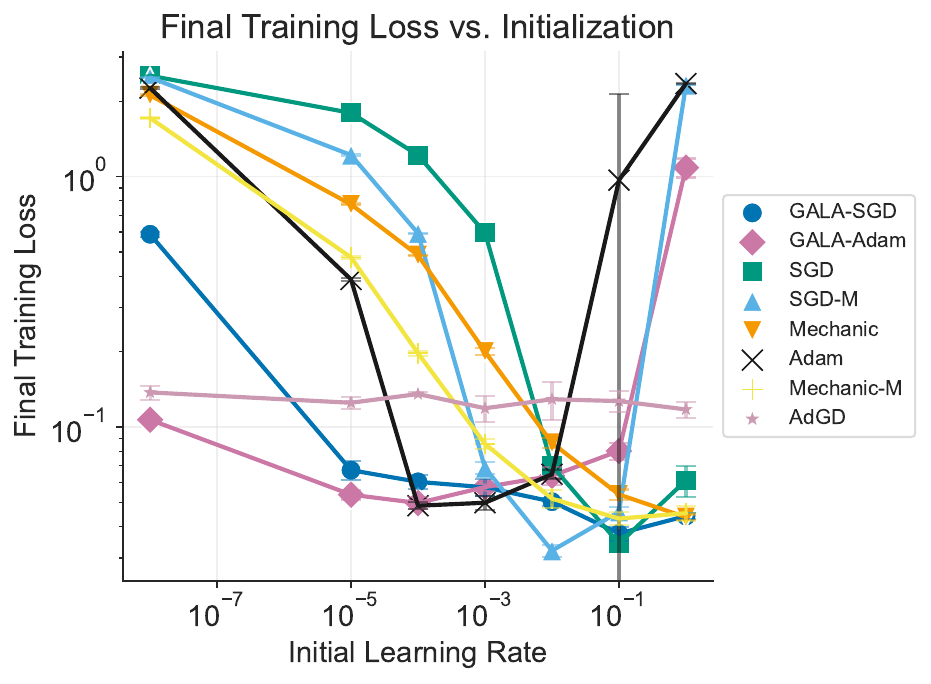}
    \caption{CIFAR-10}
    \label{fig:app-cifar10-trainloss}
    \end{subfigure}
    \begin{subfigure}{0.33\linewidth}
    \includegraphics[width=0.98\linewidth]{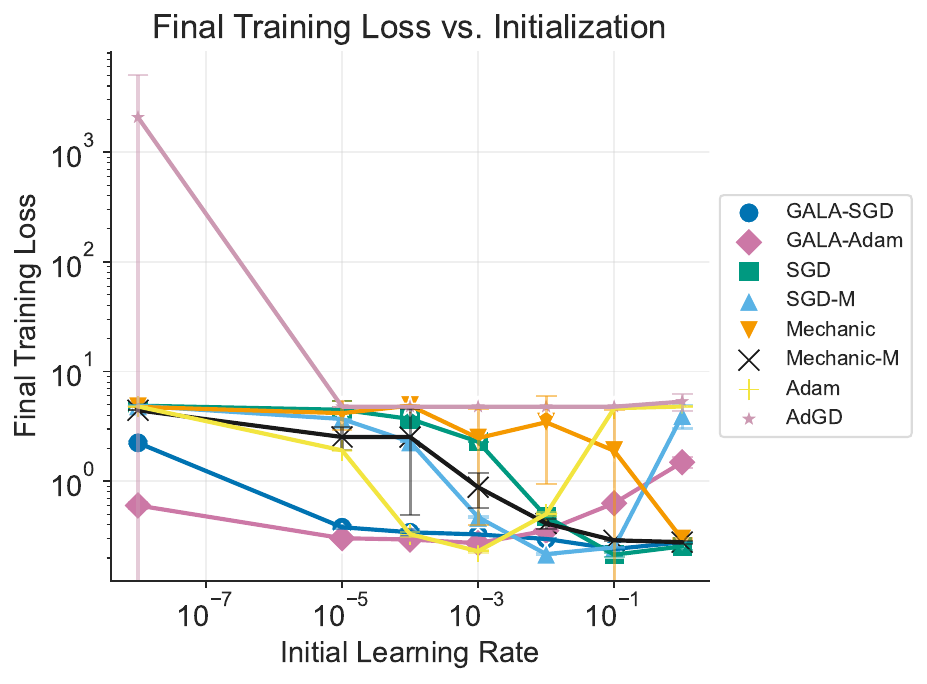}
    \caption{CIFAR-100}
    \label{fig:app-cifar100-trainloss}
    \end{subfigure}
    \begin{subfigure}{0.33\linewidth}
    \includegraphics[width=0.98\linewidth]{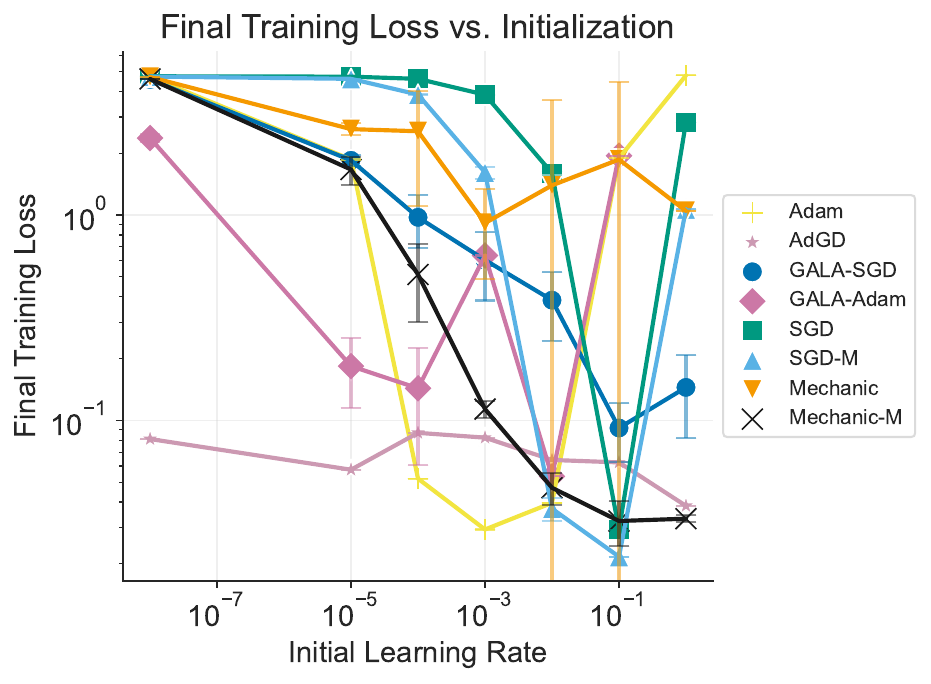}
    \caption{Flowers102}
    \label{fig:app-flowers102-trainloss}
    \end{subfigure}
    \caption{Comparison of \textbf{final training loss} values obtained from different initial learning rates for the CIFAR-10, CIFAR-100, and Flower102 datasets. We compare the performance of \sgdmethod, \adammethod against SGD, Adam, AdGD, and Mechanic. We initialize each algorithm with learning rates $[1, 10^{-1}, 10^{-2}, 10^{-3}, 10^{-4}, 10^{-5}, 10^{-8}]$ and execute 3 seeded runs for each.}
    \label{fig:app-trainloss}
\end{figure}
\begin{figure}[h]
    \begin{subfigure}{0.33\linewidth}
    \includegraphics[width=0.98\linewidth]{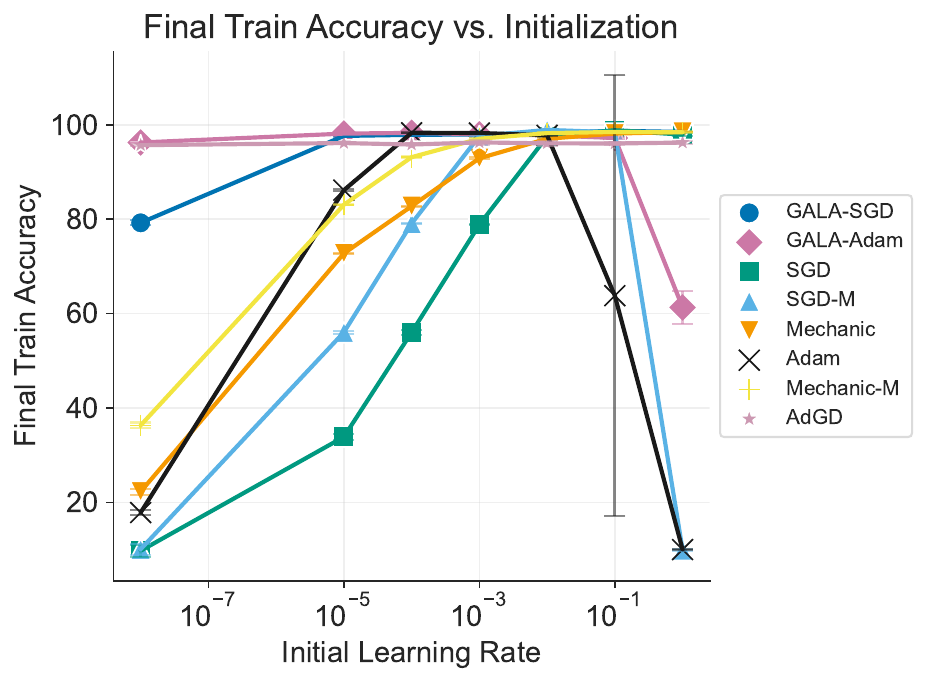}
    \caption{CIFAR-10}
    \label{fig:app-cifar10-trainacc}
    \end{subfigure}
    \begin{subfigure}{0.33\linewidth}
    \includegraphics[width=0.98\linewidth]{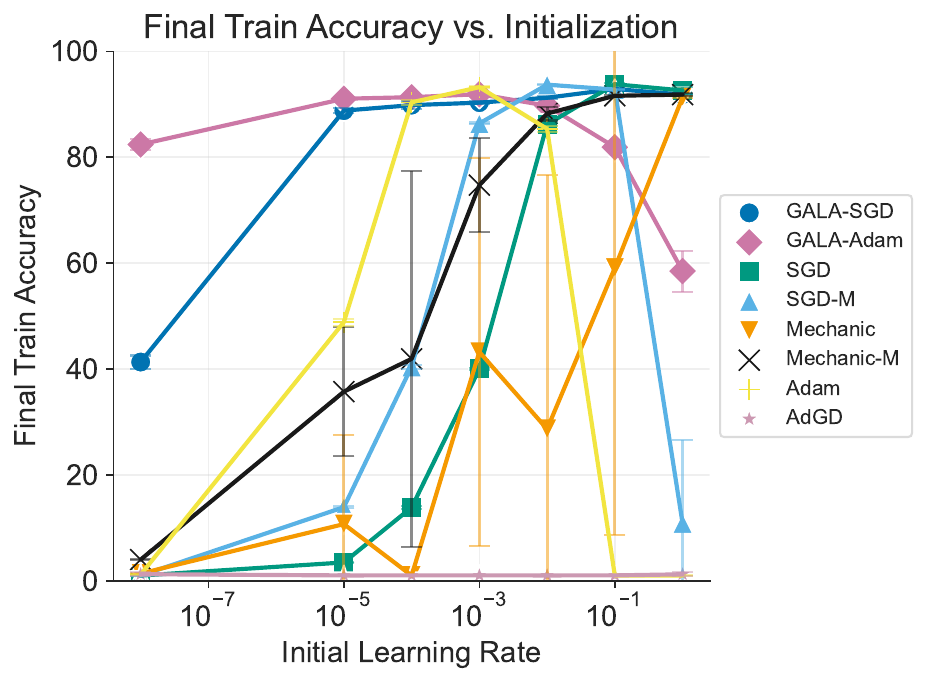}
    \caption{CIFAR-100}
    \label{fig:app-cifar100-trainacc}
    \end{subfigure}
    \begin{subfigure}{0.33\linewidth}
    \includegraphics[width=0.98\linewidth]{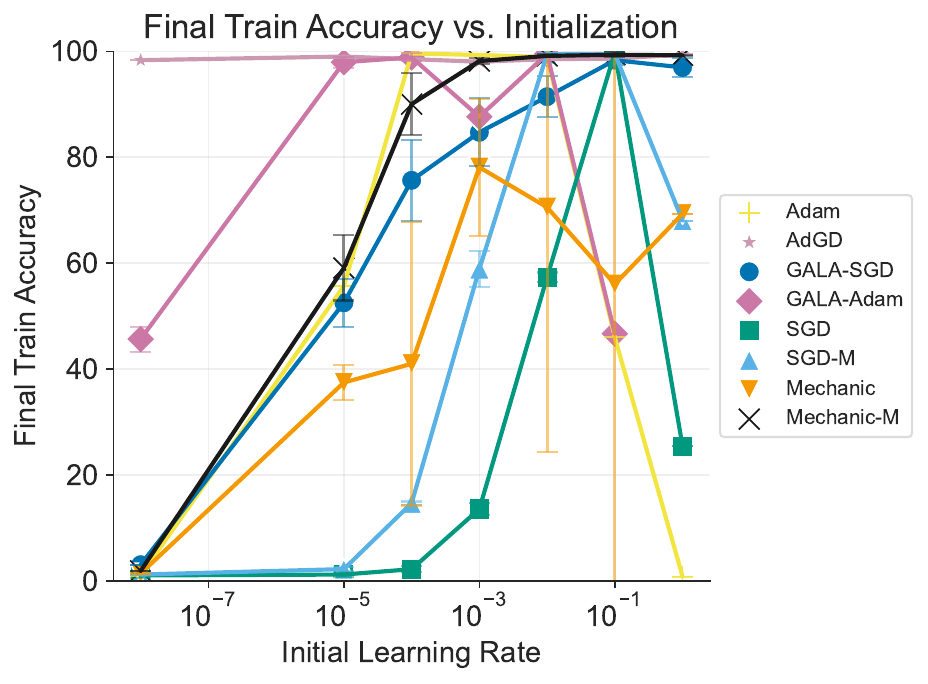}
    \caption{Flowers102}
    \label{fig:app-flowers102-trainacc}
    \end{subfigure}
    \caption{Comparison of \textbf{final training accuracy} obtained from different initial learning rates for  the CIFAR-10, CIFAR-100, and Flower102 datasets. We compare the performance of \sgdmethod, \adammethod against SGD, Adam, AdGD, and Mechanic. We initialize each algorithm with learning rates $[1, 10^{-1}, 10^{-2}, 10^{-3}, 10^{-4}, 10^{-5}, 10^{-8}]$ and execute 3 seeded runs for each.}
    \label{fig:app-trainacc}
\end{figure}
\begin{figure}[h]
    \begin{subfigure}{0.33\linewidth}
    \includegraphics[width=0.98\linewidth]{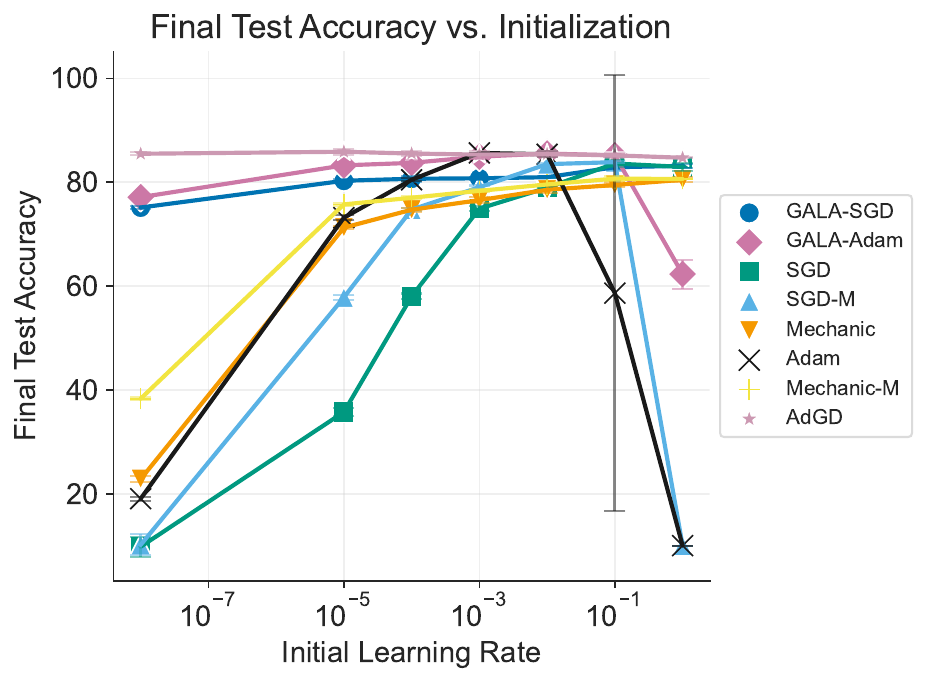}
    \caption{CIFAR-10}
    \label{fig:app-cifar10-testacc}
    \end{subfigure}
    \begin{subfigure}{0.33\linewidth}
    \includegraphics[width=0.98\linewidth]{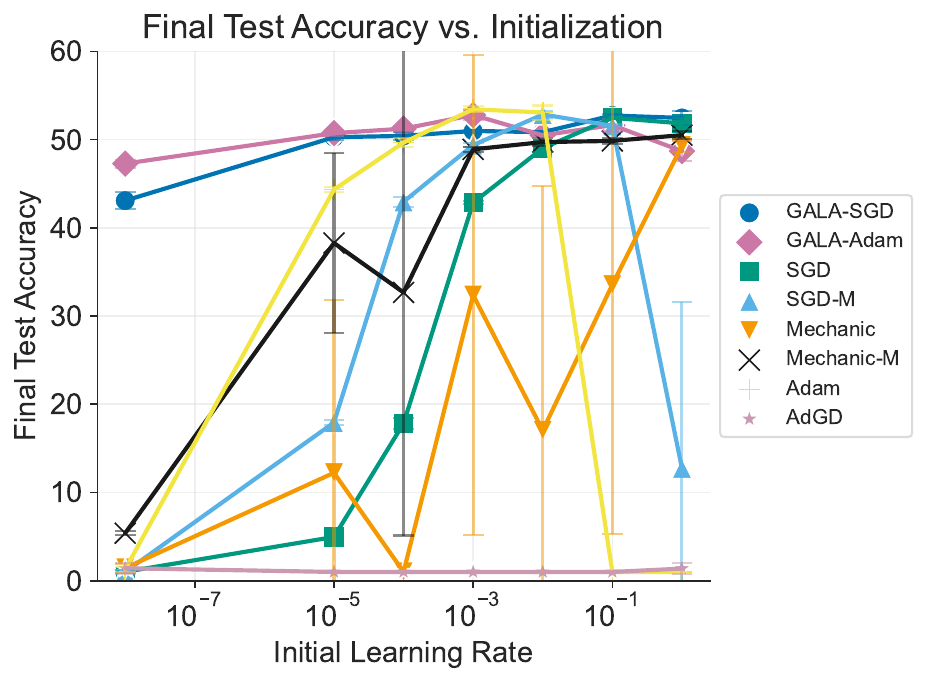}
    \caption{CIFAR-100}
    \label{fig:app-cifar100-testacc}
    \end{subfigure}
    \begin{subfigure}{0.33\linewidth}
    \includegraphics[width=0.98\linewidth]{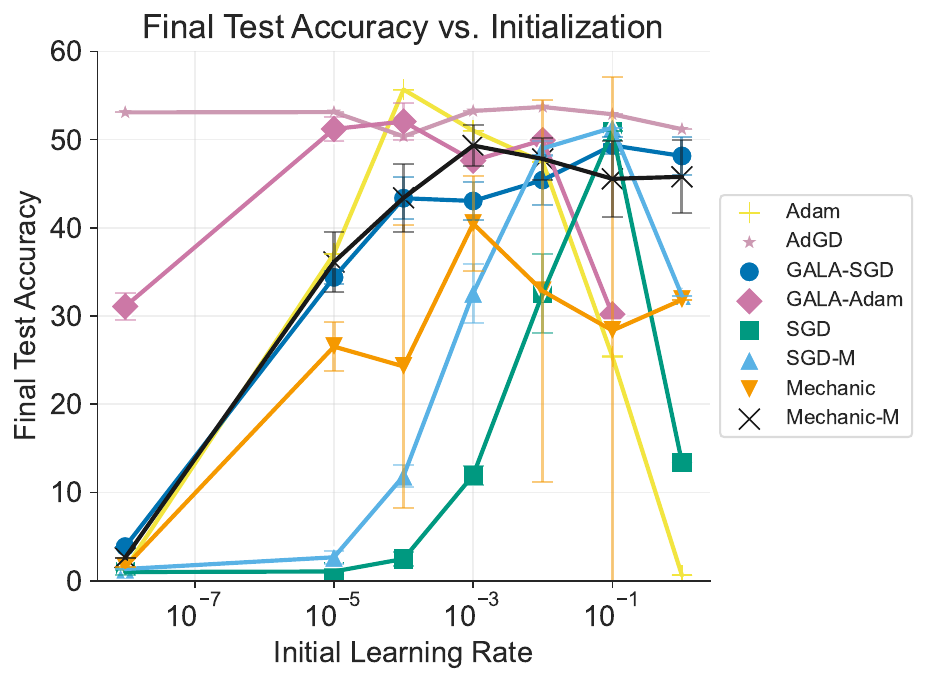}
    \caption{Flowers102}
    \label{fig:app-flowers102-testacc}
    \end{subfigure}
    \caption{Comparison of \textbf{final test accuracy} values obtained from different initial learning rates for the CIFAR-10, CIFAR-100, and Flower102 datasets. We compare the performance of \sgdmethod, \adammethod against SGD, Adam, AdGD, and Mechanic. We initialize each algorithm with learning rates $[1, 10^{-1}, 10^{-2}, 10^{-3}, 10^{-4}, 10^{-5}, 10^{-8}]$ and execute 3 seeded runs for each.}
    \label{fig:app-testacc}
\end{figure}

Across the three different datasets, we observe that our methods, \sgdmethod and \adammethod, remain robust with respect to the initial learning rates and has negligible variance due to random seeds. Moreover, they are competitive to the best-performing method in all the experiments. Compared to Mechanic and AdGD, the variance for different seeds is smaller for \sgdmethod and \adammethod. Among the GALA-variants, \sgdmethod performs better than \adammethod for larger learning rates. On the other hand, SGD, SGD with momentum, and Adam are particularly sensitive to the choice of the initial learning rate.

The results on all three datasets (although more pronounced for CIFAR-10 and CIFAR-100) show that Mechanic tends to perform better with larger learning rates, but there is higher variance with different random seeds. 
Interestingly, AdGD fails on CIFAR-100 dataset; as we will discuss later in more detail over the learning rate evolution of the method, this is likely due to the fact that its learning rates become too large in some scenarios. In fact its practical implementation requires an additional mechanism to limit the growth of the learning rate, introducing an extra hyperparameter. To foster stability, a particular choice is recommended for the hyperparameter~\citep{malitsky2020adaptive}. 

Overall, we empirically validate that \method extends the operating window of its base optimizer; it consistently improves stability for very small and relatively large initial learning rates.

\begin{figure}[ht!]
    \begin{subfigure}{0.33\linewidth}
    \includegraphics[width=0.98\linewidth]{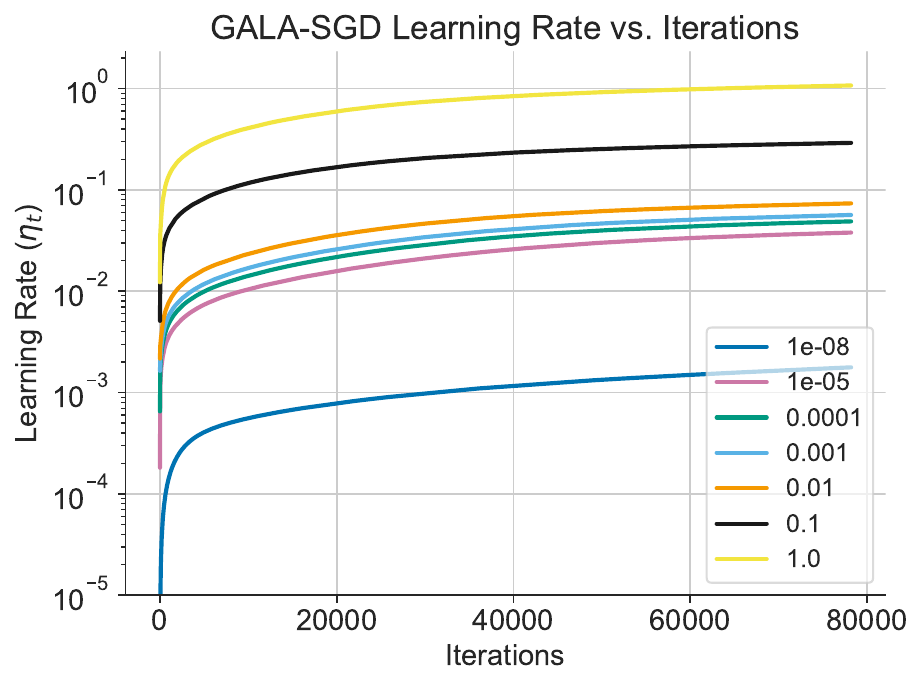}
    \caption{\sgdmethod}
    \label{fig:app-galasgd-lr}
    \end{subfigure}
    \begin{subfigure}{0.33\linewidth}
    \includegraphics[width=0.98\linewidth]{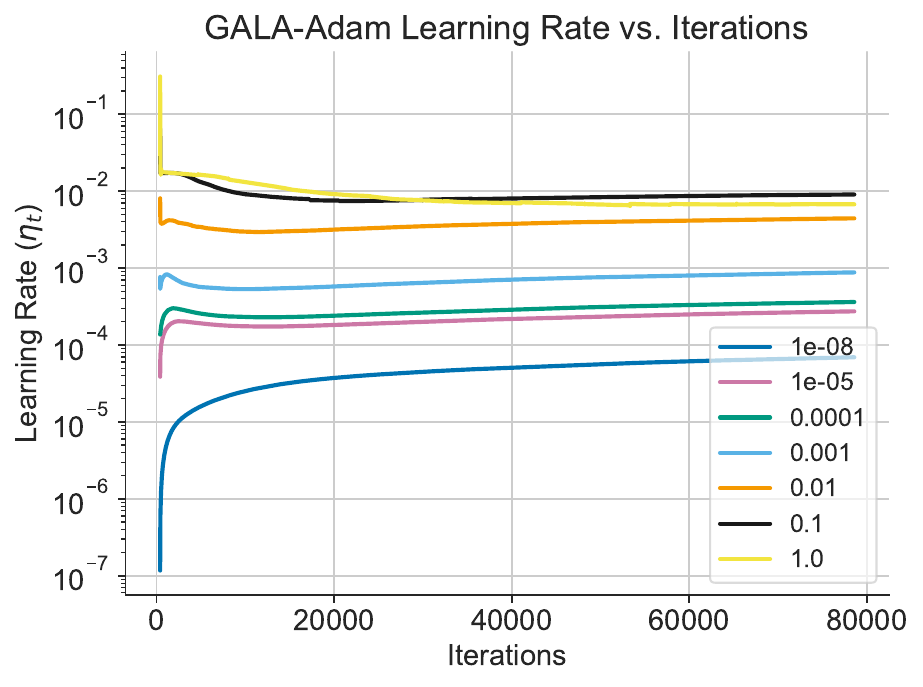}
    \caption{\adammethod}
    \label{fig:app-galaadam-lr}
    \end{subfigure}
    \begin{subfigure}{0.33\linewidth}
    \includegraphics[width=0.98\linewidth]{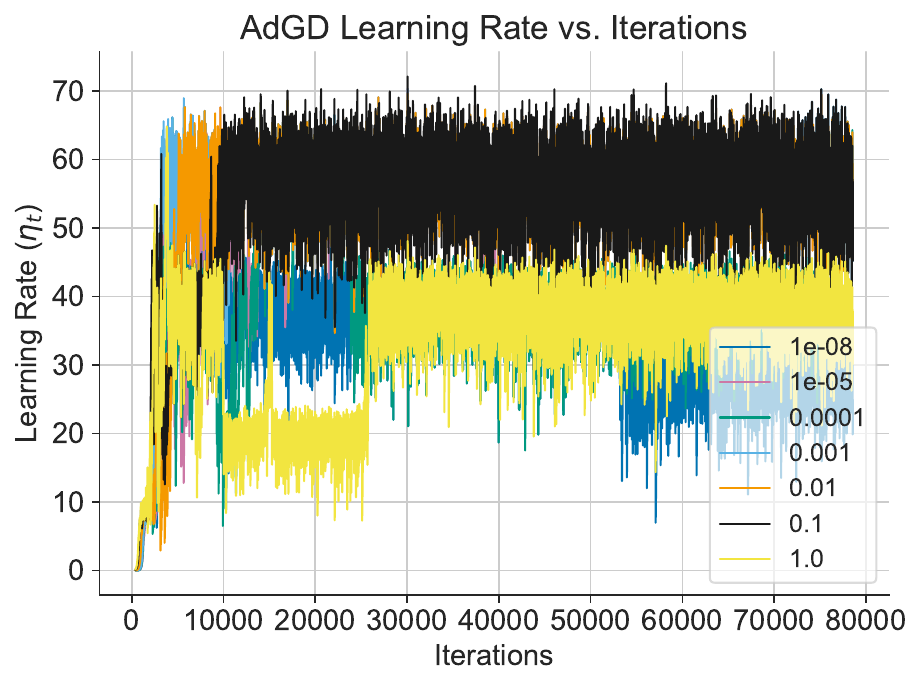}
    \caption{AdGD}
    \label{fig:app-adgd-lr}
    \end{subfigure}
    \caption{Comparison of \textbf{learning rate evolution} for \sgdmethod, \adammethod and AdGD on the CIFAR-100 dataset, averaged over 3 runs.}
    \label{fig:app-lr}
    \vspace{-1em}
\end{figure}
\paragraph{Learning rate evolution}
To better understand the convergence behavior of our method, we visualize the learning rate dynamics during optimization
in \cref{fig:app-lr} on the CIFAR-100 dataset. As shown in \cref{fig:app-galasgd-lr}, the learning rate of \sgdmethod evolves similarly and converges to similar values across a wide range of initialization, excluding extreme cases such as $\eta = 1$, $0.1$, or $10^{-8}$. This convergence likely explains the robustness of \sgdmethod to the choice of initial learning rate. Also, in \cref{fig:app-galaadam-lr}, we observe that, depending on the initial value, the learning rate of \adammethod adapts over time and can both increase and decrease. In most cases, the learning rate stabilizes between $10^{-2}$ and $10^{-3}$, which roughly corresponds to the best fixed learning rate for Adam according to \cref{fig:app-cifar100-testacc}. By contrast, \cref{fig:app-adgd-lr} shows that the learning rate chosen by AdGD tends to oscillate and frequently becomes excessively large, which may contribute to its degraded performance. While AdGD performs competitively on CIFAR-10, its behavior on CIFAR-100 suggests that it may be less robust and that the hyperparameter $\alpha$ in \eqref{eq:adgd} may require retuning for stable performance on new datasets.

\section{Conclusion}\label{sec:conclusion}

In this paper, we proposed a principled framework, \method, that dynamically adjusts the learning rate based on gradient alignment and a local curvature estimate. Motivated by convergence analysis, we formulated learning rate selection as a one-dimensional online learning problem and solve it using an online learning algorithm. We established convergence guarantees for normalized SGD equipped with \method and conduct preliminary experiments demonstrating that, when combined with SGD or Adam, our method yields robust performance across a wide range of initial learning rates.
One potential limitation of our work is that the convergence analysis is established for one instantiation of \method and our experiments focus on its integration with SGD and Adam. An interesting future venue is to extend our framework to a broader class of optimizers.

\section*{Acknowledgments}
This work is supported in part by NSF Grant CCF-2007668, the NSF AI Institute for Foundations of Machine Learning (IFML), and the Wireless Networking and Communications Group (WNCG) Industrial Affiliates Program at UT Austin.
Research of Ali Kavis is funded in part by the Swiss National Science Foundation (SNSF) under grant number P500PT\_217942. We are grateful for computing support on the Vista GPU Cluster through the Center for Generative AI (CGAI) and the Texas Advanced Computing Center (TACC) at the University of Texas at Austin.

\printbibliography

\newpage
\appendix

\section*{Appendix}

\section{Proof of Lemma~\ref{lem:regret}}\label{appen:regret_SGD}
Recall from~\eqref{eq:integral} that $F(\vx_{t+1}) - F(\vx_t) = -\eta_t\E_{\lambda_t,\xi_t'}[\langle \vg_t'(\vw_t), \vg_t(\vx_t)\rangle]$. 
We now decompose the right-hand side as $\langle \vg_t'(\vw_t), \vg_t(\vx_t)\rangle = \langle \vg_t'(\vx_t), \vg_t(\vx_t)\rangle + \langle \vg_t'(\vw_t) - \vg_t'(\vx_t) , \vg_t(\vx_t) \rangle$, where $\vg_t'(\vx_t) = \nabla f(\vx_t; \xi_t)$. For the first term, since $\xi_t$ and $\xi'$ are independent samples from the distribution $\mathcal{D}$, it holds that $\E[ \langle \vg_t'(\vx_t), \vg_t(\vx_t)\rangle] = \E[\|\nabla F(\vx_t)\|^2]$. Moreover, for the second term, it follows from Cauchy-Schwarz inequality and the definition of $L_t$ that
\begin{equation*}
    \langle\vg_t'(\vw_t) - \vg_t'(\vx_t) , \vg_t(\vx_t) \rangle \geq -\|\vg_t'(\vw_t) - \vg_t'(\vx_t)\| \|\vg_t(\vx_t)\| = -L_t \|\vw_t - \vx_t\| \|\vg_t(\vx_t)\|. 
\end{equation*}  
Since $\vw_t = \vx_{t} + \lambda_t(\vx_{t+1}-\vx_t)$ and $\lambda_t \in [0,1]$, we further have $\|\vw_t-\vx_t\| = \lambda_t \|\vx_{t+1} - \vx_t\| \leq \|\vx_{t+1} - \vx_t\| = \eta_t \|\vg_t(\vx_t)\|$, which leads to $\langle\vg_t'(\vw_t) - \vg_t'(\vx_t) , \vg_t(\vx_t) \rangle \geq -L_t \eta_t \|\vg_t(\vx_t)\|^2$. By combining both results, we obtain that 
\begin{equation*}
    \E[\langle \vg_t'(\vw_t), \vg_t(\vx_t)\rangle]  \geq \E[\|\nabla F(\vx_t)\|^2] - \E[L_t \eta_t \|\vg_t(\vx_t)\|^2].
\end{equation*} 
Hence, taking expectations on both sides of \eqref{eq:integral}, we can write  
\begin{align}
    \E[F(\vx_{t+1}) - F(\vx_t)] &= -\E[\eta_t\langle \vg_t'(\vw_t), \vg_t(\vx_t)\rangle] \nonumber\\
    &= -\E[(\eta_t - \eta)\langle \vg_t'(\vw_t), \vg_t(\vx_t)\rangle] \!-\! \eta\E[\langle \vg_t'(\vw_t), \vg_t(\vx_t)\rangle] \nonumber\\
    &\leq \!-\E[(\eta_t - \eta)\langle \vg_t'(\vw_t), \vg_t(\vx_t)\rangle] \!-\! \eta\E[\|\nabla F(\vx_t)\|^2] \!+\! \eta \E[L_t \eta_t \|\vg_t(\vx_t)\|^2]. \label{eq:before_young}
\end{align}
Moreover, from Young's inequality $\eta \eta_t \leq \frac{\eta^2}{2} + \frac{\eta_t^2}{2}$, the last term in \eqref{eq:before_young} can be bounded by $L_t \eta \eta_t \|\vg_t(\vx_t)\|^2 \leq \frac{L_t \vg_t(\vx_t)\|^2 \eta^2}{2} + \frac{L_t \vg_t(\vx_t)\|^2 \eta_t^2}{2}$. Thus, we obtain 
\begin{align}
    \E[F(\vx_{t+1}) - F(\vx_t)] &\leq - \eta\E[\|\nabla F(\vx_t)\|^2]  -\E[(\eta_t - \eta)\langle \vg_t'(\vw_t), \vg_t(\vx_t)\rangle] + \E[L_t  \|\vg_t(\vx_t)\|^2 \eta^2] \nonumber\\
    &\phantom{{}\leq{}} + \frac{L_t \vg_t(\vx_t)\|^2 \eta_t^2}{2} - \frac{L_t \vg_t(\vx_t)\|^2 \eta^2}{2} \nonumber\\
    & =  - \eta\E[\|\nabla F(\vx_t)\|^2] + \eta^2\E[L_t  \|\vg_t(\vx_t)\|^2 ] + \E[\ell_t(\eta_t) - \ell_t(\eta)], \label{eq:one_step_bound}
\end{align}
where in the last equality we used the definition of the surrogate loss function in \eqref{eq:surrogate_loss}.  
Moreover, Since $L_t \leq L^{\max}$ for any $t\geq 0$ with probability one, we have $\E[L_t \|\vg_t(\vx_t)\|^2]\leq L^{\max}\E[ \|\vg_t(\vx_t)\|^2] \leq L^{\max} (\E[\|\nabla F(\vx_t)\|^2] + \sigma^2)$. Plugging this bound in \eqref{eq:one_step_bound} and rearranging, we obtain 
\begin{equation*}
    \E[(\eta - \eta^2 L^{\max})\|\nabla F(\vx_t)\|^2] \leq \E[F(\vx_t) - F(\vx_{t+1})] + L^{\max} \eta^2\sigma^2 + \E[\ell_t(\eta_t) - \ell_t(\eta)]. 
\end{equation*}
Summing the above inequality from $t=0$ to $t=T-1$ yields~\eqref{eq:sum_grad_squares}. This completes the proof.

\section{Proof of Theorem~\ref{thm:NGD}}\label{appen:NGD}

We divide the proof of Theorem~\ref{thm:NGD} into the following three steps.

\textbf{Step 1:} Following similar arguments as in the proof of Lemma~\ref{lem:regret}, we first bound the function value decrease after one iteration. Its proof can be found in \cref{appen:one_step_improv}.   
\begin{lemma}\label{lem:one_step_improv}
    For any $\eta >0$, we have  $\E [F(\vx_{t+1}) - F(\vx_t)] \leq  \E\Bigl[- \frac{\eta}{3}\|\nabla F(\vx_t)\| + L_t \eta^2 + \frac{8 \eta}{3}\|\vm_t - \nabla F(\vx_t)\|\Bigr] + \E\Bigl[- (\eta_t -\eta) \Bigl\langle \vg'(\vw_t), \frac{\vm_t}{\|\vm_t\|}\Bigr\rangle + \frac{L_t}{2}(\eta_t^2-\eta^2)\Bigr] %
    $. 
\end{lemma}
In the above bound, the first bracketed term shows up in the analysis of normalized SGD with momentum in~\cite{cutkosky2020momentum}; it is the upper bound we get when choosing $\eta_t = \eta$. 
Moreover, the second term in the bracket captures the difference between the actual learning rate $\eta_t$ and the comparator $\eta$. It will be incorporated into the surrogate loss function and be bounded by the regret. 
\textbf{Step 2:} Next, we controls the approximation error $\E[\|\vm_t - \nabla F(\vx_t)\|]$ incurred by exponential moving averaging.
\begin{lemma}~\label{lem:ema}
    Define $\tilde{L}_t = \frac{\|g_t'(\vx_{t+1}) - g_t'(\vx_{t})\|}{\|\vx_{t+1} - \vx_t\|}$. Then we have $\sum_{t=0}^{T-1} \E[\|\vm_t - \nabla F(\vx_t)\| ] \leq \frac{\sigma}{\alpha} + \sigma \sqrt{\alpha}T + \frac{1-\alpha}{\alpha} \sum_{t=0}^{T-2} \E[\tilde{L}_{t}\eta_t]$. 
\end{lemma}
Lemma~\ref{lem:ema} upper bounds the approximation error in terms of the learning rate $\eta_t$. As we shall see in the next step, this term will also be incorporated into our surrogate loss and be bounded by the regret.

\textbf{Step 3:} 
By summing the inequality in Lemma~\ref{lem:one_step_improv} from $t=0$ to $t=T-1$ and applying 
Lemma~\ref{lem:ema}, we obtain 
\begin{align*}
    & \phantom{{}={}} \E[F(\vx_T) - F(\vx_0)] \\
     &\leq -\frac{\eta}{3} \E \Bigl[\sum_{t=0}^{T-1} \|\nabla F(\vx_t)\| \Bigr] + \E \Bigl[\sum_{t=0}^{T-1} L_t \Bigr] \eta^2 + \frac{8\eta}{3}\E\Bigl[ \sum_{t=0}^{T-1} \|\vm_t - \nabla F(\vx_t)\|\Bigr] \\
    & \phantom{{}={}} + \sum_{t=0}^{T-1} \E\Bigl[- (\eta_t -\eta) \Bigl\langle \vg'(\vw_t), \frac{\vm_t}{\|\vm_t\|}\Bigr\rangle + \frac{L_t}{2}(\eta_t^2-\eta^2)\Bigr] \\
    & \leq -\frac{\eta}{3} \E \Bigl[\sum_{t=0}^{T-1} \|\nabla F(\vx_t)\| \Bigr] + \E \Bigl[\sum_{t=0}^{T-1} L_t \Bigr] \eta^2 + \frac{8\eta}{3}(\frac{\sigma}{\alpha} + \sigma \sqrt{\alpha}T) + \frac{8(1-\alpha)}{3\alpha}\E \Bigl[ \sum_{t=0}^{T-2} \tilde{L}_t \eta_t \eta\Bigr]\\
    & \phantom{{}={}}  + \sum_{t=0}^{T-1} \E\Bigl[- (\eta_t -\eta) \Bigl\langle \vg'(\vw_t), \frac{\vm_t}{\|\vm_t\|}\Bigr\rangle + \frac{L_t}{2}(\eta_t^2-\eta^2)\Bigr].
\end{align*}
Moreover, by Young's inequality, we have $\tilde{L}_t \eta_t \eta \leq \frac{\tilde{L}_t}{2}\eta_t^2 + \frac{\tilde{L}_t}{2}\eta^2  = \tilde{L}_t \eta^2 + (\frac{\tilde{L}_t}{2}\eta_t^2 - \frac{\tilde{L}_t}{2}\eta^2)$. 
Using $\Delta_F = F(\vx_0) - F(\vx^*)  \geq F(\vx_0) - F(\vx_T)$ and recalling the definition of $\ell^{\mathrm{N}}_t$ in \eqref{eq:surrogate_loss_N}, we obtain 
\begin{align*}
    0 &\leq - \frac{\eta}{3} \sum_{t=0}^{T-1} \E[\|\nabla F(\vx_t)\|]   + \frac{8\eta}{3}(\frac{\sigma}{\alpha} + \sigma \sqrt{\alpha}T) + \E\Bigl[{\sum_{t=0}^{T-1}L_{t}} + \frac{8(1-\alpha)\sum_{t=0}^{T-1} \tilde{L}_{t}}{3\alpha} \Bigr]\eta^2 \\
    & \phantom{{}\leq{}}  + \E[\sum_{t=0}^{T-1} (\ell^{\mathrm{N}}_t(\eta_t) - \ell^{\mathrm{N}}_t(\eta))] + \E[{\Delta}_F]. 
  \end{align*}
Now for any $\eta \in [0, \eta^{\max}]$, we can upper bound $\sum_{t=0}^{T-1} (\ell^{\mathrm{N}}_t(\eta_t) - \ell^{\mathrm{N}}_t(\eta)) \leq \mathrm{Reg}_T^{\mathrm{N}}$ by definition, and hence we can choose the value of $\eta$ freely from the interval $[0, \eta^{\max}]$ in the above bound. We now consider the following cases: 
\begin{enumerate}[(i),leftmargin=2em]
  \item \textbf{Case I:} we have $\sum_{t=0}^{T-1} \E[\|\nabla F(\vx_t)\|] \leq 16(\frac{\sigma}{\alpha} + \sigma\sqrt{\alpha}T)$;
  \item \textbf{Case II:} we have $\sum_{t=0}^{T-1} \E[\|\nabla F(\vx_t)\|] \geq 16(\frac{\sigma}{\alpha} + \sigma\sqrt{\alpha}T)$. This further implies that 
  \begin{equation}\label{eq:before_opt_eta}
    0 \leq - \frac{\eta}{6} \sum_{t=0}^{T-1} \E[\|\nabla F(\vx_t)\|]  + \E\Bigl[\sum_{t=0}^{T-1}L_{t} + \frac{8 (1-\alpha)\sum_{t=0}^{T-1}\tilde{L}_{t}}{3\alpha} \Bigr]\eta^2 + \E[\mathrm{Reg}_T^{\mathrm{N}} + \Delta_F].
  \end{equation}
  Moreover, we set the value of $\eta$ as 
  \begin{equation}\label{eq:eta_value}
    \eta = \min\left\{ \frac{\sum_{t=0}^{T-1} \E[\|\nabla F(\vx_t)\|]}{\E[12\sum_{t=0}^{T-1}L_{t} +  \frac{32(1-\alpha)}{\alpha}\sum_{t=0}^{T-1} \tilde{L}_{t}]}, {\eta}^{\max} \right\}.
  \end{equation}
  This again leads to two subcases depending on the value of $\eta$: 
  \begin{itemize}
    \item If $\eta$ takes the first value in \eqref{eq:eta_value}, we obtain from \eqref{eq:before_opt_eta} that 
    \begin{equation*}
      \frac{1}{12} \frac{(\sum_{t=0}^{T-1} \E[\|\nabla F(\vx_t)\|])^2}{\E[12\sum_{t=0}^{T-1}L_{t} +  \frac{32(1-\alpha)}{\alpha}\sum_{t=0}^{T-1} \tilde{L}_{t}]} \leq \E[\mathrm{Reg}_T^{\mathrm{N}} + \Delta_F]. 
    \end{equation*}
    To simplify the notation, let $M = \Delta_F + \mathrm{Reg}_T^{\mathrm{N}}$. 
    With some algebraic manipulation and using the fact that $\sqrt{a+b} \leq \sqrt{a} + \sqrt{b}$, we obtain 
\begin{equation*}
    \sum_{t=0}^{T-1} \E[\|\nabla F(\vx_t)\|] \leq 12 \sqrt{\E[M] \E\Bigl[\sum_{t=0}^{T-1}L_{t}\Bigr]} + 8\sqrt{\frac{6(1-\alpha)}{\alpha}}\sqrt{\E[M]\E\Bigl[\sum_{t=0}^{T-1} \tilde{L}_{t}\Bigr]}.
\end{equation*}

    \item If $\eta$ takes the second value in \eqref{eq:eta_value}, then $\eta = {\eta}^{\max} \leq  \frac{\sum_{t=0}^{T-1} \E[\|\nabla F(\vx_t)\|]}{\E[12\sum_{t=0}^{T-1}L_{t} +  \frac{32(1-\alpha)}{\alpha}\sum_{t=0}^{T-1} \tilde{L}_{t}]}$. In this case, we obtain from \eqref{eq:before_opt_eta} that 
    \begin{equation*}
      \frac{\eta^{\max}}{12}\sum_{t=0}^{T-1} \E[\|\nabla F(\vx_t)\|] \leq 
      \E[M] \quad \Rightarrow \quad \sum_{t=0}^{T-1} \E[\|\nabla F(\vx_t)\|] \leq \frac{12\E[M]}{\eta^{\max}}.  
    \end{equation*}
  \end{itemize}
\end{enumerate}
Combining the upper bounds in all cases and using the definition of $L_T^{\mathrm{avg}}$, we can deduce that 
\begin{equation}\label{eq:before_choosing_alpha}
  \sum_{t=0}^{T-1} \E[\|\nabla F(\vx_t)\|] \leq 16(\frac{\sigma}{\alpha} + \sigma\sqrt{\alpha}T) \!+\!12 \sqrt{\E[M] L_T^{\mathrm{avg}} T} \!+\! 8\sqrt{\frac{6(1-\alpha)}{\alpha}}\sqrt{\E[M]L_T^{\mathrm{avg}} T}+ \frac{12\E[M]}{\eta^{\max}}. %
\end{equation}
Finally, we can choose the parameter $\alpha$ to optimize the above upper bound. Specifically, we let  
\begin{equation*}
  \alpha = \min\{ \frac{\sqrt{L_T^{\mathrm{avg}}\E[M] }}{\sigma\sqrt{T}},1\}.
\end{equation*}
If $\frac{\sqrt{L_T^{\mathrm{avg}}\E[M] }}{\sigma\sqrt{T}} \leq 1$, then we have $\frac{16\sigma}{\alpha} \leq \frac{16\sigma^2\sqrt{T}}{\sqrt{L_T^{\mathrm{avg}}\E[M]}}$, $16 \sigma \sqrt{\alpha}T \leq 16{\sigma^{1/2}(L^{\mathrm{avg}}_T\E[M])^{1/4}}{T^{3/4}}$, and $8\sqrt{\frac{6(1-\alpha)}{\alpha}}\sqrt{\E[M]L_T^{\mathrm{avg}} T} \leq 8\sqrt{6} {\sigma^{1/2}(L^{\mathrm{avg}}_T\E[M])^{1/4}}{T^{3/4}}$.  Otherwise, if $\frac{\sqrt{L_T^{\mathrm{avg}}\E[M] }}{\sigma\sqrt{T}} > 1$, then $\alpha = 1$ and $\frac{16\sigma}{\alpha} = 16 \sigma \leq \frac{16\sqrt{L_T^{\mathrm{avg}}\E[M] }}{\sqrt{T}}$, $16 \sigma \sqrt{\alpha}T \leq 16 \sigma T \leq 16 {\sqrt{L_T^{\mathrm{avg}}\E[M] T}}$, and $8\sqrt{\frac{6(1-\alpha)}{\alpha}}\sqrt{\E[M]L_T^{\mathrm{avg}} T} = 0$.   Hence, combining both cases, we conclude that 
\begin{align*}
    & \phantom{{}={}}16(\frac{\sigma}{\alpha} + \sigma\sqrt{\alpha}T)+ 8\sqrt{\frac{6(1-\alpha)}{\alpha}}\sqrt{\E[M]L_T^{\mathrm{avg}} T} \\
    &\leq \frac{16\sigma^2\sqrt{T}}{\sqrt{L_T^{\mathrm{avg}}\E[M]}} + (16+8\sqrt{6}){\sigma^{1/2}(L^{\mathrm{avg}}_T\E[M])^{1/4}}{T^{3/4}} + 32{\sqrt{L_T^{\mathrm{avg}}\E[M] T}}.
\end{align*}
By using the above bound and dividing both sides by $T$ in \eqref{eq:before_choosing_alpha}, we arrive at
\begin{align*}
  \frac{1}{T}\sum_{t=0}^{T-1} \E[\|\nabla F(\vx_t)\|] &\leq \frac{(16+8\sqrt{6}){\sigma^{1/2}(L^{\mathrm{avg}}_T\E[M])^{1/4}}}{T^{1/4}} + \frac{16\sigma^2}{\sqrt{L_T^{\mathrm{avg}}\E[M]T}} + 44\frac{\sqrt{L_T^{\mathrm{avg}}\E[M]}}{\sqrt{T}} \\
  & \phantom{{}\leq{}} + \frac{12\E[M]}{\eta^{\max} T}. 
\end{align*}
This completes the proof of Theorem~\ref{thm:NGD}.

\subsection{Proof of Lemma~\ref{lem:one_step_improv}}\label{appen:one_step_improv}
Similar to the arguments in \cref{sec:online_learning}, we first apply the fundamental theorem of calculus to get $F(\vx_{t+1}) - F(\vx_t) = \langle \bm{\nabla}_t, \vx_{t+1} - \vx_t \rangle = -\eta_t \langle \bm{\nabla}_t, \frac{\vm_t}{\|\vm_t\|} \rangle$. Since $\bm{\nabla}_t = \E_{\lambda_t}[\nabla F(\vw_t)] = \E_{\lambda_t, \xi_t'}[\vg_t'(\vw_t)]$, we further have 
\begin{equation}\label{eq:FTC-NGD}
    F(\vx_{t+1}) - F(\vx_t) = - \eta_t \E_{\lambda_t, \xi_t'} \Bigl[ \Bigl\langle \vg_t'(\vw_t), \frac{\vm_t}{\|\vm_t\|}\Bigr\rangle \Bigr]. 
\end{equation} 
Next, we decompose the right-hand side of \eqref{eq:FTC-NGD} as 
\begin{align*}
    \Bigl\langle \vg_t'(\vw_t), \frac{\vm_t}{\|\vm_t\|}\Bigr\rangle &= \Bigl\langle \vg_t'(\vx_t), \frac{\vm_t}{\|\vm_t\|}\Bigr\rangle +  \Bigl\langle \vg_t'(\vw_t) - \vg_t'(\vx_t), \frac{\vm_t}{\|\vm_t\|}\Bigr\rangle \\
    &\geq  \Bigl\langle \vg_t'(\vx_t), \frac{\vm_t}{\|\vm_t\|}\Bigr\rangle - \|\vg_t'(\vw_t) - \vg_t'(\vx_t)\|, 
\end{align*}
where we used Cauchy-Schwarz inequality in the last step. Using the definition of $L_t$, we have 
\begin{equation}\label{eq:Lt-NGD}
    \|\vg_t'(\vw_t) - \vg_t'(\vx_t)\| \leq L_t \|\vw_t-\vx_t\| = L_t\lambda_t \|\vx_{t+1} - \vx_t\| \leq L_t \eta_t. 
\end{equation}
Moreover,  since $\vg_t'(\vx_t)$ and $\vm_t$ are independent conditioned on $\vx_t$, we further have $\E[\langle \vg_t'(\vx_t), \frac{\vm_t}{\|\vm_t\|}\rangle] = \E[\langle \nabla F(\vx_t), \frac{\vm_t}{\|\vm_t\|}\rangle]$, which is further lower bounded in the following lemma. 
\begin{lemma}\label{lem:normalized_align}
    We have $\langle \nabla F(\vx_t), \frac{\vm_t}{\|\vm_t\|}\rangle \geq \frac{1}{3} \|\nabla F(\vx_t)\| - \frac{8}{3}\|\vm_t - \nabla F(\vx_t)\|$. 
\end{lemma}
\begin{proof}
    Our proof is inspired by \cite[Lemma 2]{cutkosky2020momentum}. We consider two cases: 
    \begin{enumerate}[(i)]
        \item If $\|{\vm}_t - \nabla F(\vx_t)\| \leq \frac{1}{2}\|\nabla F(\vx_t)\|$, then by the triangle inequality, we have $\|{\vm}_t\| \leq \frac{3}{2}\|\nabla F(\vx_t)\|$. Therefore, we have 
    \begin{align*}
      \frac{1}{\|\vm_t\|}\langle \nabla F(\vx_t), {\vm}_t \rangle &= \frac{1}{\|\vm_t\|}(\|\nabla F(\vx_t)\|^2 + \langle \nabla F(\vx_t), \vm_t - \nabla F(\vx_t)\rangle ) \\
      &\geq \frac{1}{\|\vm_t\|}(\|\nabla F(\vx_t)\|^2 - \frac{1}{2}\|\nabla F(\vx_t)\|^2)  \geq \frac{1}{3}\|\nabla F(\vx_t)\|,
    \end{align*}
    where we used $\|{\vm}_t - \nabla F(\vx_t)\| \leq \frac{1}{2}\|\nabla F(\vx_t)\|$ in the first inequality and $\|{\vm}_t\| \leq \frac{3}{2}\|\nabla F(\vx_t)\|$ in the second one. Since $\frac{8}{3}\|\vm_t - \nabla F(\vx_t)\| \geq 0$, the result in Lemma~\ref{lem:normalized_align} holds under this case. 
    \item Otherwise, if $\|{\vm}_t - \nabla F(\vx_t)\| > \frac{1}{2}\|\nabla F(\vx_t)\|$, we can instead use Cauchy-Schwarz inequality to bound 
    \begin{align*}
      \frac{1}{\|\vm_t\|}\langle \nabla F(\vx_t), {\vm}_t \rangle \geq -\|\nabla F(\vx_t)\| &= \frac{1}{3}\|\nabla F(\vx_t)\| - \frac{4}{3}\|\nabla F(\vx_t)\| \\ &\geq \frac{1}{3}\|\nabla F(\vx_t)\| - \frac{8}{3}\|{\vm}_t - \nabla F(\vx_t)\|,
    \end{align*}
    where we used $\|{\vm}_t - \nabla F(\vx_t)\| > \frac{1}{2}\|\nabla F(\vx_t)\|$ in the last inequality. 
    \end{enumerate}
    This completes the proof.
\end{proof}

Combining \eqref{eq:Lt-NGD} and Lemma~\ref{lem:normalized_align}, we obtain that 
\begin{equation*}
    \E \Bigl[ \Bigl\langle \vg_t'(\vw_t), \frac{\vm_t}{\|\vm_t\|}\Bigr\rangle \Bigr] \geq \E \Bigl[ \frac{1}{3}\|\nabla F(\vx_t)\| - \frac{8}{3}\|{\vm}_t - \nabla F(\vx_t)\| - L_t \eta_t \Bigr]. 
\end{equation*}
Hence, it further follows from \eqref{eq:FTC-NGD} that 
\begin{align*}
   \E[ F(\vx_{t+1}) - F(\vx_t) ]  &= -\E \Bigl [(\eta_t - \eta) \Bigl\langle \vg_t'(\vw_t), \frac{\vm_t}{\|\vm_t\|} \Bigr\rangle \Bigr ] - \eta \E \Bigl[ \Bigl\langle \vg_t'(\vw_t), \frac{\vm_t}{\|\vm_t\|}\Bigr\rangle \Bigr] \\
   &\leq -\E \Bigl [(\eta_t - \eta) \Bigl\langle \vg_t'(\vw_t), \frac{\vm_t}{\|\vm_t\|} \Bigr\rangle  - \frac{\eta \|\nabla F(\vx_t)\|}{3} + \frac{8\eta \|{\vm}_t - \nabla F(\vx_t)\|}{3}  + L_t \eta_t \eta\Bigr ].
\end{align*}
By using Young's inequality $\eta_t \eta \leq \frac{\eta_t^2}{2} + \frac{\eta^2}{2}$ and rearranging, we obtain the inequality in \cref{lem:one_step_improv}. 

\subsection{Proof of Lemma~\ref{lem:ema}}
  From the update rule in \eqref{eq:NGD}, we can write
  \begin{equation}\label{eq:resursive_m}
    \begin{aligned}
    \vm_t - \nabla F(\vx_t) &= (1-\alpha) (\vm_{t-1} - \nabla F(\vx_{t-1})) +\alpha(\nabla f(\vx_t; \xi_t) - \nabla F(\vx_t)) \\ 
    &\phantom{{}={}}+ (1-\alpha)(\nabla F(\vx_{t-1}) - \nabla F(\vx_t)). %
  \end{aligned}
  \end{equation}  
  Define the stochastic gradient error $\ve_t = \nabla f(\vx_t; \xi_t) - \nabla F(\vx_t)$. By Assumption~\ref{assum:variance}, we have $\E[\ve_t] = 0$ and $\E[\|\ve_t\|^2] \leq \sigma^2$. Moreover, by multiplying both sides of \eqref{eq:resursive_m} with $(1-\alpha)^{-t}$, we have 
  \begin{align*}
    (\vm_t - \nabla F(\vx_t)) (1-\alpha)^{-t} &= (\vm_{t-1} - \nabla F(\vx_{t-1})) (1-\alpha)^{-t+1} + \alpha\ve_t  (1-\alpha)^{-t} \\
    &\phantom{{}={}}+ (\nabla F(\vx_{t-1}) - \nabla F(\vx_t)) (1-\alpha)^{-t+1}.
  \end{align*} 
  Note that we set $\vm_0 = \nabla f(\vx_0; \xi_0)$. Thus, by summing the above inequality, we obtain 
  \begin{align*}
    (\vm_t - \nabla F(\vx_t)) (1-\alpha)^{-t} &= \ve_0 + \sum_{s=1}^t \ve_s \alpha(1-\alpha)^{-s} 
    + \sum_{s=1}^t (\nabla F(\vx_{s-1}) - \nabla F(\vx_s)) (1-\alpha)^{-s+1}.
  \end{align*}
  Therefore, it follows from the triangle inequality that
  \begin{equation}\label{eq:after_triangle}
    \|\vm_t - \nabla F(\vx_t)\| \leq \|\ve_0\| (1-\alpha)^t + \left\|\sum_{s=1}^t \ve_s \alpha(1-\alpha)^{t-s}  \right\| + \sum_{s = 1}^t \|\nabla F(\vx_{s-1}) - \nabla F(\vx_s)\|  (1-\alpha)^{t-s+1}.
  \end{equation}
  By Jensen's inequality and the fact that $\{\xi_s\}_{s=1}^t$ are i.i.d. sampled from $\mathcal{D}$, we have $\E[\|\ve_0\|] \leq \sqrt{\E[\|\ve_0\|^2]} = \sigma$ and 
  \begin{equation*}
    \E\left\|\sum_{s=1}^t \ve_s \alpha(1-\alpha)^{t-s}  \right\| \leq \sqrt{\E \Bigl\|\sum_{s=1}^t \ve_s \alpha(1-\alpha)^{t-s} \Bigr\|^2} \leq \sqrt{\sum_{s=1}^t \sigma^2 \alpha^2(1-\alpha)^{2(t-s)}}. 
  \end{equation*}
  Moreover, it also follows from Jensen's inequality that $\E[\|\nabla F(\vx_{s-1}) - \nabla F(\vx_s)\| ]\leq \E[\|\nabla f(\vx_{s}; \xi_{s}) - \nabla f(\vx_{s-1}; \xi_{s})\|] = \tilde{L}_{s-1}\|\vx_{s} - \vx_{s-1}\| = \tilde{L}_{s-1} \eta_{s-1}$. Hence, by taking the expectation on both sides of \eqref{eq:after_triangle}, we further have
  \begin{align*}
    \E[\|\vm_t - \nabla F(\vx_t)\| ] &\leq \sigma (1-\alpha)^t + \sigma \alpha\sqrt{\sum_{s=1}^t (1-\alpha)^{2(t-s)}} + \sum_{s = 1}^{t}  \E[\tilde{L}_{s-1}\eta_{s-1}] (1-\alpha)^{t-s+1} \\
    & \leq \sigma (1-\alpha)^t + \sigma \alpha \sqrt{\frac{1}{1-(1-\alpha)^2}} + \sum_{s = 0}^{t-1}  \E[\tilde{L}_{s}\eta_{s}](1-\alpha)^{t-s} \\
    & \leq \sigma (1-\alpha)^t + \sigma \sqrt{\alpha} + \sum_{s = 0}^{t-1}  \E[\tilde{L}_{s}\eta_{s}] (1-\alpha)^{t-s}.
  \end{align*}
By summing the above inequality from $t=0$ to $t=T-1$, we obtain that 
\begin{equation*}
    \sum_{t=0}^{T-1} \E[\|\vm_t - \nabla F(\vx_t)\| ] \leq \sigma \sum_{t=0}^{T-1} (1-\alpha)^t  + \sigma \sqrt{\alpha} T + \sum_{t=0}^{T-1}\sum_{s = 0}^{t-1}  \E[\tilde{L}_{s}\eta_{s}] (1-\alpha)^{t-s}
\end{equation*}
Since $\sum_{t=0}^{T-1} (1-\alpha)^t \leq \frac{1}{\alpha}$ and $\sum_{t=0}^{T-1}\sum_{s = 0}^{t-1}  \E[\tilde{L}_{s}\eta_{s}] (1-\alpha)^{t-s} = \sum_{s=0}^{T-2}\sum_{t = s+1}^{T-1}  \E[\tilde{L}_{s}\eta_{s}] (1-\alpha)^{t-s} \leq \frac{1-\alpha}{\alpha} \sum_{s=0}^{T-2} \E[\tilde{L}_{s}\eta_{s}] $, we obtain \cref{lem:ema}.  

\section{Proof of Lemma~\ref{lem:regret_bound}}\label{appen:regret}
As discussed in \cref{sec:convergence}, our update for $\eta$ can be viewed as an instance of the optimistic FTRL algorithm. Therefore, we can inovke the convergence bound in \cite[Theorem 7.39]{Orabona2019}, where $\psi_1 = \dots =\psi_T = \frac{\delta}{2}\eta^2$ and $\tilde{\ell}_{t+1}(\eta) = -\eta \langle \vg_{t+1}(\vx_{t+1}), \frac{\vm_{t+1}}{\|\vm_{t+1}\|}  \rangle$. Moreover, note that $\frac{\delta}{2}\eta^2 + \sum_{s=0}^t \ell^{\mathrm{N}}_t(\eta)$ is $(\delta + \sum_{s=0}^t (L_s + \frac{8(1-\alpha)}{3\alpha}\tilde{L}_s))$-strongly convex, and $|(\ell_t^{\mathrm{N}})'(\eta_t) - \tilde{\ell}'_t(\eta_t)| = |-\langle \vg_t'(\vw_t) - \vg_t(\vx_t), \frac{\vm_t}{\|\vm_t\|}  \rangle + L_t \eta_t + \frac{8(1-\alpha)}{3\alpha}\tilde{L}_t \eta_t| \leq \|\vg_t'(\vw_t) - \vg_t(\vx_t)\| + L_t \eta_t + \frac{8(1-\alpha)}{3\alpha}\tilde{L}_t \eta_t$. Hence, we have 
  \begin{equation}\label{eq:regret_OFTRL}
    \sum_{t=0}^{T-1} (\ell_t(\eta_t) - \ell_t(\eta) )\leq \frac{\delta}{2}\eta^2 + \sum_{t=0}^{T-1} \frac{(L_{t} \eta_t + \frac{8 (1-\alpha) }{3\alpha}\tilde{L}_{t}\eta_t+ \|\vg_t'(\vw_t) - \vg_t(\vx_t)\|)^2}{2\delta+2\sum_{s=0}^{t} (L_{s} + \frac{8 (1-\alpha) \tilde{L}_{s}}{3\alpha})}. %
  \end{equation}
Moreover, by the triangle inequality and the definition of $L_t$, we have $\|\vg_t'(\vw_t) - \vg_t(\vx_t)\| = \|\nabla f(\vw_t; \xi_{t}') - \nabla f(\vx_t; \xi_t)\| \leq \|\nabla f(\vw_t; \xi_{t}') - \nabla f(\vx_t; \xi_t')\| + \|\nabla f(\vx_t; \xi_t') - \nabla f(\vx_t; \xi_t)\| \leq L_t \eta_t + \|\nabla f(\vx_t; \xi_t')- \nabla f(\vx_t; \xi_t)\|$. So we can further bound the summand in \eqref{eq:regret_OFTRL} by  
\begin{align*}
    & \phantom{{}={}} \frac{(2L_t \eta_t + \frac{8(1-\alpha)}{3\alpha}\tilde{L}_t \eta_t + \|\nabla f(\vx_t; \xi_t')- \nabla f(\vx_t; \xi_t)\|)^2}{2\delta+2\sum_{s=0}^{t} (L_{s} + \frac{8 (1-\alpha) \tilde{L}_{s}}{3\alpha})} \\
    & \leq \frac{{\eta}_t^2(2L_t + \frac{8(1-\alpha)}{3\alpha}\tilde{L}_t)^2 + \|\nabla f(\vx_t; \xi_t')- \nabla f(\vx_t; \xi_t)\|^2}{\delta+\sum_{s=0}^{t} (L_{s} + \frac{8 (1-\alpha) \tilde{L}_{s}}{3\alpha})}. 
\end{align*} 
In the following, we will upper bound the two sums 
\begin{equation*}
    \sum_{t=0}^{T-1} \frac{{\eta}_t^2(2L_t + \frac{8(1-\alpha)}{3\alpha}\tilde{L}_t)^2}{\delta+\sum_{s=0}^{t} (L_{s} + \frac{8 (1-\alpha) \tilde{L}_{s}}{3\alpha})} \quad \text{and} \quad \sum_{t=0}^{T-1} \frac{\|\nabla f(\vx_t; \xi_t')- \nabla f(\vx_t; \xi_t)\|^2}{\delta+\sum_{s=0}^{t} (L_{s} + \frac{8 (1-\alpha) \tilde{L}_{s}}{3\alpha})}
\end{equation*}
separately.

By our assumption,  $\max\{L_t, \tilde{L}_t\} \leq L^{\max}$ with probability one and $\eta_t \leq \eta^{\max}$. Thus, we can derive 
\begin{equation}\label{eq:before_sum_to_log}
    \sum_{t=0}^{T-1} \frac{{\eta}_t^2(2L_t + \frac{8(1-\alpha)}{3\alpha}\tilde{L}_t)^2}{\delta+\sum_{s=0}^{t} (L_{s} + \frac{8 (1-\alpha) \tilde{L}_{s}}{3\alpha})} \leq \frac{28(\eta^{\max})^2 L^{\max}}{3\alpha} \sum_{t=0}^{T-1} \frac{L_{t} + \frac{8 (1-\alpha) \tilde{L}_{t}}{3\alpha}}{\delta+\sum_{s=0}^{t} (L_{s} + \frac{8 (1-\alpha) \tilde{L}_{s}}{3\alpha})}.
\end{equation}
Now we can apply the following lemma. 
\begin{lemma}\label{lem:sum_to_log}
    For any nonnegative sequence $\{a_t\}_{t=0}^{T-1}$ and $\delta > 0$, it holds that $\sum_{t=0}^{T-1} \frac{a_t}{\delta + \sum_{s=0}^t a_s} \leq \log\left(1+\frac{\sum_{t=0}^{T-1}a_t}{\delta}\right)$. 
\end{lemma}
\begin{proof}
    For any $t\geq 0$, we have $\frac{a_t}{\delta + \sum_{s=0}^t a_s} = 1 - \frac{\delta+ \sum_{s=0}^{t-1} a_s}{\delta + \sum_{s=0}^t a_s} \leq \log(\frac{\delta + \sum_{s=0}^t a_s}{\delta+ \sum_{s=0}^{t-1} a_s})$, where we used the fact that $1-x \leq \log(\frac{1}{x})$ for any $x \geq 0$. Hence, by summing the inequality from $t=0$ to $t=T-1$ we obtain $\sum_{t=0}^{T-1} \frac{a_t}{\delta + \sum_{s=0}^t a_s} \leq \log (\frac{\delta  + \sum_{t=0}^{T-1} a_t}{\delta})  = \log(1+\frac{\sum_{t=0}^{T-1} a_t}{\delta})$. 
\end{proof}
Hence, by applying \cref{lem:sum_to_log} to \eqref{eq:before_sum_to_log}, we get 
\begin{align*}
    \sum_{t=0}^{T-1} \frac{{\eta}_t^2(2L_t + \frac{8(1-\alpha)}{3\alpha}\tilde{L}_t)^2}{\delta+\sum_{s=0}^{t} (L_{s} + \frac{8 (1-\alpha) \tilde{L}_{s}}{3\alpha})} &\leq\frac{28(\eta^{\max})^2 L^{\max}}{3\alpha} \log\Bigl( 1+ \frac{\sum_{t=0}^{T-1} (L_{t} + \frac{8 (1-\alpha) \tilde{L}_{t}}{3\alpha})}{\delta} \Bigr) \\
    &\leq \frac{28(\eta^{\max})^2 L^{\max}}{3\alpha} \log\Bigl( 1+ \frac{ 11 L^{\max}}{3\alpha\delta}T \Bigr). 
\end{align*} 
By our choice of $\eta^{\max} = \sqrt{\alpha} \bar{\eta}$, it becomes $\bigO\left(\bar{\eta}^2 L^{\max} \log\left(1+\frac{L^{\max} }{\alpha \delta}T\right)\right)$. 
For the second term, by our assumption, we have $\frac{1}{t+1} \sum_{s=0}^t L_s \geq M^{\mathrm{avg}}$ with probability one. Furthermore, using Assumption~\ref{assum:variance}, we have 
\begin{align*}
    \E \Bigl[ \sum_{t=0}^{T-1} \frac{\|\nabla f(\vx_t; \xi_t')- \nabla f(\vx_t; \xi_t)\|^2}{\delta+\sum_{s=0}^{t} (L_{s} + \frac{8 (1-\alpha) \tilde{L}_{s}}{3\alpha})} \Bigr] &\leq \sum_{t=0}^{T-1} \frac{\E [\|\nabla f(\vx_t; \xi_t')- \nabla f(\vx_t; \xi_t)\|^2]}{M^{\mathrm{avg}} (t+1)} \\
    & = \sum_{t=0}^{T-1} \frac{2 \sigma^2}{M^{\mathrm{avg}} (t+1)} \leq \frac{2\sigma^2}{M^{\mathrm{avg}}}(1+ \log (T)).
\end{align*}
\cref{lem:regret_bound} now follows from combining the above two bounds. 

\end{document}